\documentclass{article}

\usepackage{amsmath,amssymb,amsthm,natbib}
\usepackage{hyperref,xcolor,booktabs}
\usepackage{graphicx,float}
\usepackage{makecell}
\usepackage{algorithm}
\usepackage{algpseudocode}
\usepackage{tikz}
\usepackage{subcaption}

\usepackage{changes}

\definechangesauthor[name={Jisui}, color=purple]{js}

\usepackage{comment}

\usetikzlibrary{arrows.meta, positioning, shapes.geometric, decorations.pathreplacing}
\title{PTL-Diffusion: Manifold-Aware Diffusion with Periodic Terminal Laws}

\author{
Danqi Zhuang\textsuperscript{1},
Jisui Huang\textsuperscript{2},
Xiaoyue Xi\textsuperscript{3},
Andrew Kiggins\textsuperscript{4},\\
Xiaojie Wang\textsuperscript{5,6},
Ke Chen\textsuperscript{1,6},
Yue Wu\textsuperscript{1,6}
}

\date{}

\newtheorem{theorem}{Theorem}
\newtheorem{lemma}{Lemma}

\begin{document}

\maketitle
\footnotetext[1]{Department of Mathematics and Statistics, University of Strathclyde, Glasgow, UK}
\footnotetext[2]{Centre for Mathematical Imaging Techniques and Department	
of Mathematical Sciences, University of Liverpool, Liverpool, UK}
\footnotetext[3]{Department of Medical Statistics, London School of Hygiene \& Tropical Medicine, London, UK}
\footnotetext[4]{Unaffiliated; Email:\texttt{andrewkiggins2000@gmail.com}}
\footnotetext[5]{School of Mathematics and Statistics, HNP-LAMA, Central South University, Changsha Hunan, PR China}
\footnotetext[6]{%
\parbox[t]{0.95\textwidth}{%
\raggedright
Joint corresponding authors: 
x.j.wang7@csu.edu.cn; 
k.chen@strath.ac.uk; 
yue.wu@strath.ac.uk.
}%
}

\begin{abstract}
Standard diffusion models typically use a single time-homogeneous Gaussian terminal distribution as the reference law for generation. While this choice is analytically convenient and empirically powerful, it provides little explicit structure for data concentrated near low-dimensional manifolds, where different regions of the data distribution may correspond to distinct local geometric or semantic factors. As a result, the reverse model must recover manifold-level structure almost entirely from an unstructured terminal reference distribution.

We propose \emph{PTL-Diffusion}, a proof-of-concept diffusion framework whose forward noising process converges to a nonconstant periodic family of Gaussian terminal laws rather than to a single invariant Gaussian law. The phase variable acts as a coarse coordinate for organizing variation in the data distribution, and the resulting terminal family provides a structured reference geometry for the reverse process. Unlike a phase-conditioned DDPM, where phase information only enters the denoising network while the forward process remains unchanged, PTL-Diffusion embeds phase structure directly into the forward noising dynamics.

The proposed construction preserves much of the tractability of standard denoising diffusion models. For a periodically forced Ornstein--Uhlenbeck-type forward process, we derive closed-form forward marginals, identify the limiting periodic Gaussian terminal family, and obtain explicit Gaussian reverse posteriors. These formulae allow training with a standard noise-prediction objective. We further introduce an invariant-average regularization term that couples the phase-conditioned reverse dynamics through the averaged periodic reference law.

Experiments on synthetic torus and cylinder point-cloud benchmarks, together with the Olivetti face dataset, provide proof-of-concept evidence that periodic terminal laws can improve manifold-level distributional matching compared with a DDPM baseline under matched denoising architectures. In particular, PTL-Diffusion reduces phase-conditioned errors, feature-space covariance errors, and nearest-neighbour manifold distances. These results suggest that structured terminal reference laws are a promising direction for diffusion models on manifold-supported data, while also highlighting the need for more expressive phase constructions and larger-scale evaluations.

\end{abstract}
\section{Introduction}

Generative models play an important role in applications such as image style
translation, text-to-image generation, domain adaptation, and data
augmentation. Among many generative models, diffusion models
\citep{sohl2015deep,ho2020denoising,song2021score} have become a central
class of methods. They define a forward noising process that gradually destroys
data structure and transforms the data distribution into a simple terminal
reference law. The reverse model is then trained to map samples from this
reference law back to the data distribution. Conditional diffusion models can
further guide generation by passing class labels, text prompts, or other
auxiliary variables to the denoising network
\citep{ho2020denoising,dhariwal2021diffusion,ho2022classifier}.

Many high-dimensional datasets encountered in generative modelling are not
distributed throughout the ambient Euclidean space. Images of faces, physical
states, and stochastic trajectories often concentrate near low-dimensional
manifolds, whose coordinates may encode factors such as identity,
illumination, pose, phase, or latent dynamical state. This viewpoint is
classical in manifold learning and representation learning
\citep{turk1991eigenfaces,tenenbaum2000global,roweis2000nonlinear,
belkin2003laplacian,coifman2006diffusion,narayanan2010sample,
bengio2013representation}. A central challenge for diffusion models is
therefore not only to learn a high-dimensional distribution, but also to
recover the geometry of the data manifold while generating realistic samples.

In most standard diffusion constructions, the terminal reference law is a
single time-homogeneous Gaussian distribution. This choice is analytically
convenient and empirically powerful, but it is geometrically unstructured: it
does not retain coarse information about the manifold on which the data are
concentrated. Consequently, the reverse process must recover manifold-level
structure almost entirely from an essentially homogeneous noise distribution.
This can be restrictive for data supported near curved, cyclic, or multi-region
manifolds, where different regions of the distribution may correspond to
different local geometric or semantic factors.

From the viewpoint of manifold theory, a complex manifold is not generally
described by one global Euclidean coordinate system; rather, it can be
understood through local charts, each of which is Euclidean-like. This
observation motivates a simple question for diffusion modelling: instead of
using one homogeneous Gaussian terminal law, can we use a structured family of
Gaussian terminal laws as a coarse reference geometry for different regions or
phases of the data manifold? We use this chart intuition only as motivation:
our method does not explicitly learn a manifold atlas, transition maps, or a
Riemannian metric. Rather, it introduces a finite phase-indexed terminal
reference family that can act as a tractable coarse approximation to such
geometric organization.

Recent work has begun to address the interaction between score-based
generative modelling and non-Euclidean geometry. For example, Riemannian
score-based generative models formulate the diffusion process directly on a
manifold rather than in the ambient Euclidean space
\citep{debortoli2022riemannian}. More recent manifold-aware generative models
use normalizing flows or diffusion/score-based models to learn distributions
supported on nonlinear spaces
\citep{mathieu2020riemannian,debortoli2022riemannian,huang2022riemannian}.
Our approach is different. We remain in the ambient Euclidean space, but modify
the terminal reference structure of the forward process itself. Instead of
forcing the forward dynamics to converge to a single invariant Gaussian law, we
construct a forward process whose long-time law is a structured periodic family
of Gaussian measures.

We propose \emph{Periodic Terminal Law Diffusion} (PTL-Diffusion), a
proof-of-concept diffusion framework whose forward noising dynamics converge
to a nonconstant $P$-periodic Gaussian family. The phase variable is used as a
coarse coordinate for organizing variation in the data distribution. For
datasets with intrinsic cyclic geometry, such as torus or cylinder point-cloud
samples, the phase can be chosen from the corresponding angular coordinate.
For static image data such as faces, there is no canonical physical phase;
instead, we construct a phase index from a low-dimensional descriptor, such as
an eigenface coordinate or an illumination statistic. In both cases, the phase
should be interpreted as a coarse organizational coordinate, not as an
assumption that individual samples are themselves periodic. A schematic
overview of the proposed framework is shown in Figure~\ref{fig:ptl_diffusion_schematic}.
\begin{figure}[t]
\centering
\includegraphics[width=1.0\textwidth]{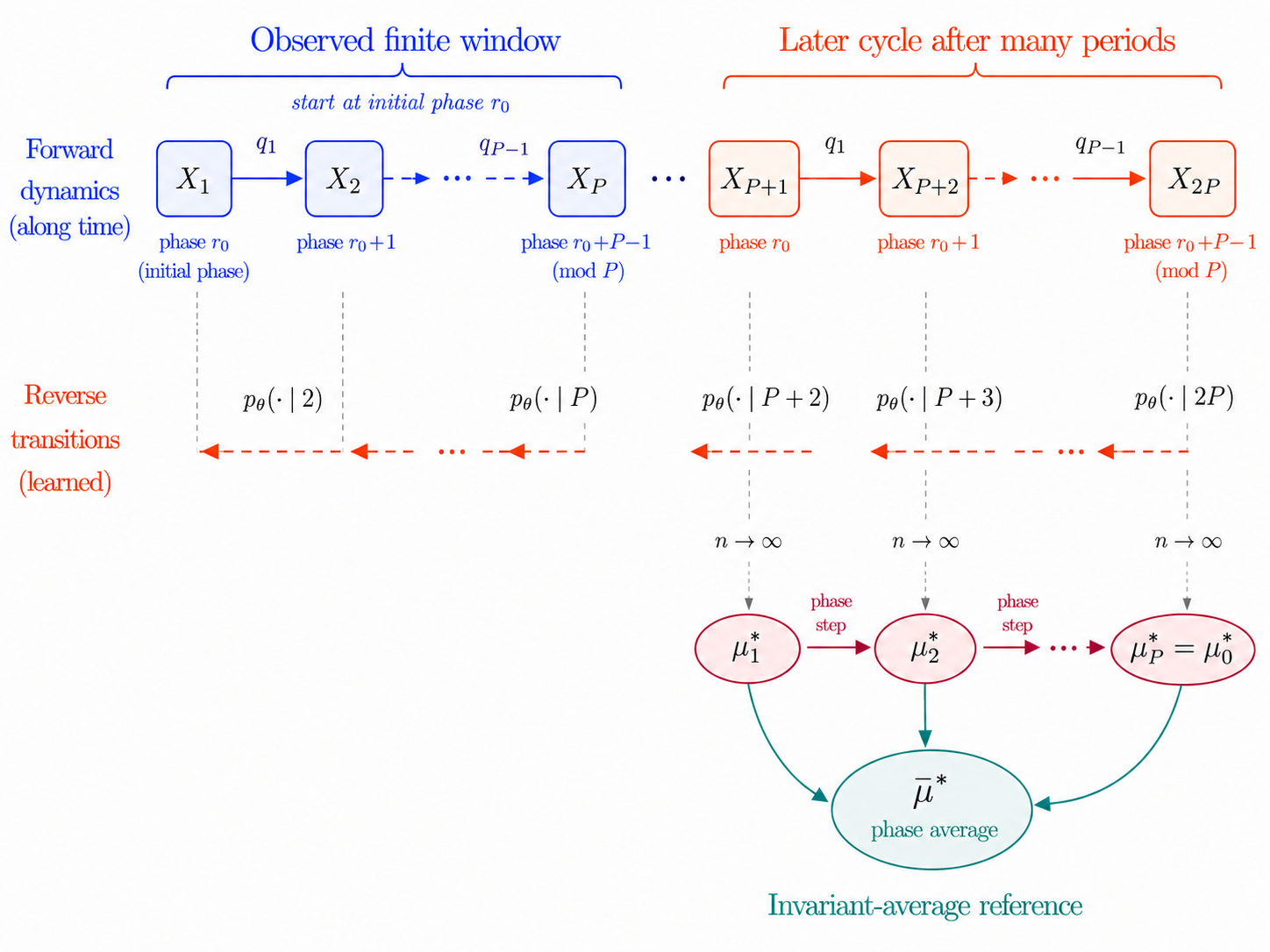}
\caption{
Schematic of PTL-Diffusion. The forward noising process is driven by a
periodic terminal reference family rather than by a single invariant Gaussian
law. The phase variable acts as a coarse coordinate for the structured terminal
family, while the reverse model learns phase-conditioned denoising dynamics
coupled by an invariant-average regularization.
}
\label{fig:ptl_diffusion_schematic}
\end{figure}

The mathematical motivation comes from the theory of random periodic solutions
and periodic measures for stochastic dynamical systems. In general, the
long-time object of a stochastic system is not necessarily a stationary
distribution, but may instead be a family of random variables or probability
measures evolving periodically in time. Random periodic solutions for
stochastic semi-flows were developed by \cite{feng2011pathwise} and further
studied for stochastic partial differential equations by
\cite{feng2012spde,feng2016anticipating}. Periodic measures and their ergodic
properties were investigated by \cite{feng2015periodic}. More recently, random
periodic behaviour has also been studied under weaker dissipativity
assumptions, including non-uniformly dissipative SDEs and functional SDEs with
finite or infinite delay \citep{bao2022random}, as well as McKean--Vlasov SDEs
and their interacting-particle approximations \citep{bao2024mckean}. Numerical approximation of random periodic solutions has been studied through
Euler--Maruyama and modified Milstein schemes for SDEs
\cite{fengliu2017numerical}, backward Euler--Maruyama schemes under monotone
drift conditions \cite{wu2023backward}, order-one backward Euler approximations
for semilinear SDEs \cite{guowangwu2025order}, and Galerkin-type exponential
integrators for semilinear stochastic evolution equations
\cite{wuyuan2023galerkin}. These works provide the conceptual background for
replacing a single stationary limiting law by a periodically evolving family of
laws.

Our use of this theory is deliberately different from directly modelling a
random periodic path or a functional stochastic differential equation. We do
not introduce memory variables or formulate the generative model as a delay
equation or an infinite-horizon stochastic integral equation. Instead, we
extract a simpler design principle: the forward noising dynamics of a diffusion
model may organize their limiting laws into a structured periodic family. In
PTL-Diffusion, this family acts as a tractable terminal reference geometry for
manifold-supported data.

PTL-Diffusion implements this principle through a periodically forced
Ornstein--Uhlenbeck-type forward process, inspired by the example illustrated in \cite{wuyuan2023galerkin}. Instead of converging to a single Gaussian terminal distribution, the forward
process converges phase-wise to a repeating family of $P$ Gaussian terminal laws,
\begin{equation}
    \{\mu_r^\ast\}_{r=0}^{P-1},
    \qquad
    \mu_{r+P}^\ast = \mu_r^\ast .
\end{equation}
Here $\mu_r^\ast$ denotes the terminal Gaussian law associated with phase $r$.
The condition $\mu_{r+P}^\ast=\mu_r^\ast$ means that the terminal family repeats
after $P$ phases. The family is nonconstant when at least two phases have
different terminal laws.

The resulting framework remains close to standard denoising diffusion models.
We derive explicit forward marginals and Gaussian reverse posteriors for the
periodically forced forward process; see Section~\ref{sec:method}. This allows
the reverse model to be trained using a standard noise-prediction objective.
Because the reverse dynamics are indexed by phase, we introduce a
phase-conditioned denoising network using circular sinusoidal phase embeddings.
We also impose an invariant-average regularization: although the model learns
phase-wise reverse dynamics, the average over phases should remain tied to a
coherent averaged reference law. This regularization couples the phase-wise
reverse models and reflects the role of the phase average as an invariant
object associated with the full periodic cycle.

Our work is related to conditional diffusion models, where class labels, text
prompts, or other auxiliary variables are supplied to the denoising network
while the forward noising process remains unchanged
\citep{ho2020denoising,dhariwal2021diffusion,ho2022classifier}. The
phase-conditioned DDPM baseline used in this paper follows this
conditioning-only principle: it receives the same phase embedding as
PTL-Diffusion, but its forward process still converges to one
time-homogeneous Gaussian terminal law. PTL-Diffusion differs in that the phase
coordinate is used to organize the forward reference law itself. Thus the
terminal object is a phase-indexed family of measures rather than a single
time-homogeneous Gaussian law. This distinction allows us to separate ordinary
phase conditioning from the proposed periodic terminal reference structure.

This paper should be read as a \emph{proof-of-concept} study. We do not claim
that a periodic terminal family is the correct reference law for all
manifold-supported data, nor that a finite phase partition can represent
arbitrary manifold geometry. The period $P$ is fixed and finite, so the
time-indexed or phase-indexed family of laws is represented through finitely
many phase classes. This can be restrictive for strongly aperiodic data,
continuously drifting nonstationary distributions, or manifolds whose latent
geometry cannot be well approximated by a coarse cyclic coordinate.
Nevertheless, the finite-period construction provides a controlled and
tractable setting in which the terminal reference law is no longer a single
invariant Gaussian, while the forward and reverse distributions remain
analytically tractable.

We evaluate PTL-Diffusion on two types of manifold-supported data. First, we
consider synthetic torus and cylinder point-cloud datasets. These examples
have explicit embedded-manifold structure and intrinsic angular coordinates,
making them controlled benchmarks for testing a phase-indexed terminal
reference law. Second, we evaluate the method on the Olivetti face dataset
\cite{olivetti_faces}, a small-scale image benchmark whose samples lie near a
low-dimensional eigenface manifold. In this case, phase is constructed from a
low-dimensional embedding of the data. Across these experiments, we compare
PTL-Diffusion with standard DDPM and phase-conditioned DDPM baselines under
matched denoising architectures and training budgets, thereby isolating the
contribution of the periodic terminal reference law from ordinary phase
conditioning. We report ambient-space metrics, manifold-feature metrics,
phase-conditioned errors, invariant-average errors, and nearest-neighbour
distances to the empirical data manifold.

Our contributions are as follows.
\begin{enumerate}
\item We introduce PTL-Diffusion, a diffusion framework whose forward process
converges to a nonconstant periodic Gaussian terminal family rather than a
single invariant terminal law.

\item We derive closed-form forward marginals and Gaussian reverse
posteriors for the periodically forced Ornstein--Uhlenbeck forward process,
preserving compatibility with standard denoising-based training
(see {Section \ref{sec:forward} and Section \ref{sec:reverse}}).
\item We introduce a phase-conditioned reverse model and an
invariant-average regularization that couples the phase-wise reverse
dynamics through the averaged reference law
({see Lemma \ref{lem:inv} and Section \ref{sec:inv}}).
\item We provide proof-of-concept experiments on manifold-supported data,
including point-cloud data on torus and cylinder geometries and a small
face-image manifold benchmark, showing that periodic terminal laws can
improve manifold-level distributional matching compared with standard Denoising Diffusion Probabilistic Model (DDPM)
and phase-conditioned DDPM baselines
({see Section \ref{sec:experiment}}).
\end{enumerate}

\section{Methodology}\label{sec:method}
This section introduces PTL-Diffusion, a diffusion framework whose forward
noising process is designed to approach a phase-indexed family of terminal
laws rather than a single time-homogeneous Gaussian reference law. We begin
with a tractable periodically forced Ornstein--Uhlenbeck-type recursion and
derive its closed-form forward marginals, periodic limiting family, and
Gaussian reverse posterior. These results show that the model remains
compatible with standard denoising-based training while replacing the terminal
reference distribution by a structured periodic family. We then introduce a
phase-conditioned reverse model, an invariant-average regularization that
couples the phase-wise reverse dynamics, and the corresponding training and
sampling algorithms. Finally, we describe how the periodic forcing can be
adapted from empirical phase centers when the data lie near a simple manifold
or admit a meaningful coarse phase coordinate.

\subsection{Periodic OU-type forward dynamics} \label{sec:forward}

In the following, all random variables are defined on a probability space
$(\Omega,\mathcal{F},\mathbb{P})$, and $\mathcal{L}(X)$ denotes the law of a
random variable $X$ under $\mathbb{P}$. We write $\mathbb{E}$ for expectation with
respect to $\mathbb{P}$. When the relevant random variables are indicated in
the subscript, for example $\mathbb{E}_{X,Y}[\cdot]$, the subscript specifies
the variables over which the expectation is taken, with all other quantities
held fixed or understood from context. Equivalently,
$\mathbb{E}_{X,Y}[\cdot]$ denotes expectation with respect to the joint law 
$\mathcal{L}(X,Y)$. 

Fix a period $P\in\mathbb{N}$. 
Given an initial phase $r_0\in\{0,\ldots,P-1\}$, we define
\begin{equation}\label{eqn:pou}
    x_{n+1}
    =
    \rho x_n
    +
    b_{r_0+n}
    +
    \sigma \varepsilon_{n+1},
\end{equation}
where $0<\rho<1$, $\sigma>0$, $(x_n)_{n} \subset \mathbb{R}^d$, $\varepsilon_n\sim\mathcal{N}(0,I_d)$ are i.i.d., $(b_n)_{n\in\mathbb{Z}}\subset\mathbb{R}^d$ is a nonconstant $P$-periodic forcing:
$$
b_{n+P}=b_n ,
$$
and all phase indices are interpreted modulo $P$.
The nonconstant assumption rules out the degenerate stationary case and ensures
that the limiting family is genuinely phase-dependent.
An example is
$
b_n=a\sin\left(2\pi n/P\right)u
$,
where $a\in\mathbb{R}$ controls the forcing amplitude and $u\in\mathbb{R}^d$ is a fixed direction.
This example is nonconstant when $P\geq 3$ and $au\neq 0$.

The one-step transition is
\begin{equation}\label{eqn:onestep_transition}
    q(x_{n+1}\mid x_n,r_0)
=
\mathcal{N}(\rho x_n+b_{r_0+n},\sigma^2 I_d).
\end{equation}
Thus the forward process remains Gaussian and tractable, as in the standard DDPM, but its drift contains an explicit periodic component.

\begin{lemma}[Forward marginal]\label{lem:forward} Consider the phase-indexed forward process defined in \eqref{eqn:pou} with initial phase $r_0$.
For every $n\geq 1$, the conditional forward marginal satisfies
\begin{equation*}
    q(x_n\mid x_0,r_0)
    =
   \mathcal{N}(m_{n,r_0}(x_0),\Sigma_n),
\end{equation*}
where
\begin{equation*}
    m_{n,r_0}(x_0)
    :=
    \rho^n x_0
    +
    \sum_{j=0}^{n-1}
    \rho^{n-1-j} b_{r_0+j},
    \qquad
    \Sigma_n
    :=
    \sigma^2
    \frac{1-\rho^{2n}}{1-\rho^2}I_d.
\end{equation*}
\end{lemma}
The proof of Lemma \ref{lem:forward} is deferred to Appendix \ref{ssec:prooflem1}. The lemma shows that the periodic forcing changes only the mean of the forward marginal. The covariance remains isotropic and admits the same closed-form geometric structure as an OU-type process. For simplicity in the experiments, we use the normalized noise scale
    $\sigma=\sqrt{1-\rho^2}$,
so that
   $ \Sigma_n=(1-\rho^{2n})I_d$.
{
\begin{theorem}[Periodic limiting family and phase-wise convergence]
\label{thm:limit}
Consider the phase-indexed forward process defined in \eqref{eqn:pou} with generic initial phase $r$.

Then there exists a unique $P$-periodic Gaussian family
\begin{equation*}
    \{\mu_r^\ast\}_{r=0}^{P-1},
    \qquad
    \mu_r^\ast=\mathcal{N}(m_r^\ast,\Sigma^\ast),
    \qquad
    \mu_{r+P}^\ast=\mu_r^\ast,
\end{equation*}
with
\begin{equation}
    m_r^\ast
    =
    \sum_{j=0}^{\infty}
    \rho^j b_{r-1-j},
    \qquad
    \Sigma^\ast
    =
    \frac{\sigma^2}{1-\rho^2}I_d .
    \label{eq:periodic_terminal_mean_cov}
\end{equation}
The family is nonconstant whenever the periodic forcing produces at least two
distinct limiting means. One forward step advances the family by one phase:
\begin{equation*}
    x_n\sim\mu_r^\ast
    \quad\Longrightarrow\quad
    x_{n+1}\sim\mu_{r+1}^\ast ,
\end{equation*}
where phase indices are understood modulo $P$.

Moreover, for any initial distribution $\nu_0$ with finite second moment and
initial phase $r_0$,
\begin{equation}
    \lim_{n\to \infty}
    \mathcal{W}_1
    \left(
        \mathcal{L}(x_n),
        \mu_{r_0+n\,(\mathrm{mod}\,P)}^\ast
    \right)
    =0,
    \label{eq:w1_phase_convergence}
\end{equation}
where, for any pair of probability measures $\mu$ and $\nu$ on $\mathbb{R}^d$ with finite first moments,  $\mathcal{W}_1$ denotes
the Wasserstein-1 distance between them  over the set of all their couplings $\Gamma(\mu,\nu)$:
\begin{equation*}
    \mathcal{W}_1(\mu,\nu)
    :=
    \inf_{\gamma\in\Gamma(\mu,\nu)}
    \int_{\mathbb{R}^d\times\mathbb{R}^d}
        |x-y|\,\gamma(\mathrm{d}x,\mathrm{d}y).
\end{equation*}
Equivalently, the law of the forward process converges along each phase subsequence to the corresponding
member of the periodic terminal family. That is, for each fixed target
phase $r\in\{0,\ldots,P-1\}$,
\begin{equation*}
    \mathcal{L}(x_{n_\ell})
    \Rightarrow
    \mu_r^\ast
    \qquad
    \text{as } \ell\to\infty ,
\end{equation*}
whenever $r_0+n_\ell\equiv r\pmod{P}$.
\end{theorem}

The proof of Theorem~\ref{thm:limit} is deferred to Appendix~\ref{ssec:proofthm1}. This theorem gives the main structural property of the proposed forward
process. Since $0<\rho<1$, the effect of the initial condition is exponentially
damped. Equivalently, one may view the limiting law by starting the recursion
far in the past and letting the starting time tend to $-\infty$; this is the
pull-back viewpoint. In this limit, the initial condition disappears, but the
periodic forcing leaves a phase-dependent contribution to the mean. The
terminal reference object is therefore not a single Gaussian law, but a
$P$-periodic family of Gaussian laws.
}
\subsection{Reverse posterior and phase-conditioned reverse model}\label{sec:reverse}

Since both the one-step transition and the forward marginal are Gaussian, the
reverse posterior is also Gaussian. In the phase-indexed forward process
\eqref{eqn:pou}, the forcing at the transition from $x_n$ to $x_{n+1}$ is
$b_{r_0+n}$, where $r_0$ is the initial phase of the data sample. Hence the
reverse posterior is conditioned on both $x_0$ and $r_0$.

\begin{lemma}[Gaussian reverse posterior]
\label{lem:reverse_posterior}
For the phase-indexed forward process \eqref{eqn:pou} with $r=r_0$ and $n\geq 1$, the reverse posterior
has the form
\begin{equation*}
    q(x_n\mid x_{n+1},x_0,r_0)
    =
    \mathcal{N}
    \left(
        \widetilde{\mu}_{n,r_0}(x_{n+1},x_0),
        \widetilde{\Sigma}_n
    \right),
\end{equation*}
where
\begin{equation*}
    \widetilde{\Sigma}_n
    =
    \sigma^2
    \frac{1-\rho^{2n}}{1-\rho^{2(n+1)}}I_d,
\end{equation*}
and
\begin{equation*}
    \widetilde{\mu}_{n,r_0}(x_{n+1},x_0)
    =
    \widetilde{\Sigma}_n
    \left(
        \Sigma_n^{-1}m_{n,r_0}(x_0)
        +
        \frac{\rho}{\sigma^2}
        \left(
            x_{n+1}-b_{r_0+n}
        \right)
    \right).
\end{equation*}

Equivalently, using the forward noise variable $\varepsilon$ defined by
\begin{equation*}
    x_{n+1}
    =
    m_{n+1,r_0}(x_0)
    +
    \sigma
    \sqrt{
        \frac{1-\rho^{2(n+1)}}{1-\rho^2}
    }
    \,\varepsilon,
    \qquad
    \varepsilon\sim\mathcal{N}(0,I_d),
\end{equation*}
the posterior mean can be written as
\begin{equation*}
    \widetilde{\mu}_{n,r_0}(x_{n+1},\varepsilon)
    =
    \frac{1}{\rho}
    \left(
        x_{n+1}-b_{r_0+n}
    \right)
    -
    \kappa_n\varepsilon,
\end{equation*}
where
\begin{equation*}
    \kappa_n
    =
    \frac{
        \sigma\sqrt{1-\rho^2}
    }{
        \rho\sqrt{1-\rho^{2(n+1)}}
    }.
\end{equation*}
\end{lemma}
This noise form is directly analogous to the parameterization used in denoising
diffusion models. It allows us to train a neural network to predict the Gaussian
noise while preserving the periodic forward structure. The proof of Lemma \ref{lem:reverse_posterior} is deferred to Appendix \ref{ssec:prooflem2}. 

{For the phase-conditioned reverse model, let
$r_n = r_0+n \ (\mathrm{mod}\, P)$
denote the phase of the noisy variable $x_n$ at diffusion depth $n$. Since the
phase variable is circular, we encode it using $ \phi(r)
    =
    \left(
        \sin\frac{2\pi r}{P},
        \cos\frac{2\pi r}{P}
    \right)$. 
The learned reverse kernel for the step from depth $n$ to depth $n-1$ is
\begin{equation*}
    p_\theta(x_{n-1}\mid x_n,n,r_n)
    =
    \mathcal{N}
    \left(
        \mu_\theta(x_n,n,r_n),
        \widetilde{\Sigma}_{n-1}
    \right),
\end{equation*}
where
\begin{equation*}
    \mu_\theta(x_n,n,r_n)
    =
    \frac{1}{\rho}
    \left(
        x_n-b_{r_n-1}
    \right)
    -
    \kappa_{n-1}
    \varepsilon_\theta(x_n,n,\phi(r_n)).
\end{equation*}
Here $r_n-1$ is understood modulo $P$, and $\kappa_{n-1}$ is the reverse-noise coefficient from Lemma~\ref{lem:reverse_posterior}.}

The phase-conditioned network learns a family of reverse dynamics indexed by
the current phase $r_n$. This is more structured than merely adding an
arbitrary time embedding: the phase corresponds to the member
$\mu_{r_n}^\ast$ of the limiting periodic family.
\subsection{Invariant-average principle}\label{sec:inv}

The phase-conditioned reverse model should not learn $P$ unrelated reverse processes. The periodic family should remain tied to a coherent averaged reference law. This motivates an invariant-average regularization.

Since the terminal family is periodic rather than stationary, no single phase
law $\mu_r^\ast$ is invariant under one forward step on $\mathbb{R}^d$. To make
the invariant object explicit, we include the phase as part of the state and
consider the lifted Markov chain on
$\{0,\ldots,P-1\}\times\mathbb{R}^d$.

{
\begin{lemma}[Invariant average on the lifted phase space]\label{lem:inv}
Let $\{\mu_r^\ast\}_{r=0}^{P-1}$ be the $P$-periodic terminal family from
Theorem~\ref{thm:limit}, and let $\{K_r\}_{r=0}^{P-1}$ be Markov kernels
satisfying
\begin{equation*}
    K_r\mu_r^\ast=\mu_{r+1}^\ast,
    \qquad r=0,\ldots,P-1,
\end{equation*}
with indices understood modulo $P$. Define the lifted Markov kernel on
$\{0,\ldots,P-1\}\times\mathbb{R}^d$ by
\begin{equation*}
    \mathcal{K}\big((r,x),\{r+1\}\times A\big)
    =
    K_r(x,A).
\end{equation*}
Then
\begin{equation*}
    \Pi^\ast
    =
    \frac{1}{P}
    \sum_{r=0}^{P-1}
    \delta_r\otimes \mu_r^\ast
\end{equation*}
is invariant under $\mathcal{K}$. Its marginal on $\mathbb{R}^d$ is
\begin{equation*}
    \bar{\mu}^\ast
    =
    \frac{1}{P}
    \sum_{r=0}^{P-1}\mu_r^\ast.
\end{equation*}
We call $\bar{\mu}^\ast$ the invariant-average reference law associated with the
periodic terminal family. It is a phase mixture on $\mathbb{R}^d$, not generally
a single Gaussian law or an invariant law for a phase-forgetting one-step kernel.
\end{lemma}
}
Thus $\bar{\mu}^\ast$ should be understood as the state-space marginal of an
invariant measure on the lifted phase space. It is not generally itself an
invariant law for a phase-forgetting one-step Markov kernel on $\mathbb{R}^d$.

This result provides the theoretical basis for the regularization term used below. It enforces that the learned periodic family has a meaningful invariant average rather than arbitrary phase-wise behavior. The proof of Lemma \ref{lem:inv} is deferred to Appendix \ref{ssec:prooflem3}. 

\subsection{Training objective and algorithms}
The preceding results give an explicit Gaussian noising process with a
phase-indexed terminal reference law. We now turn this construction into a
trainable diffusion model. The key point is that, unlike a standard DDPM whose
forward process contracts toward a single isotropic Gaussian, the PTL forward
process contracts toward a periodic family of Gaussian laws indexed by phase.
Training therefore requires two pieces of information: the diffusion depth and
the current phase of the noisy sample. The denoising network is trained to
predict the Gaussian noise in the closed-form forward marginal, while the
invariant-average term couples the phase-conditioned denoisers so that they do
not behave as unrelated models across phases.

The phase-conditioned denoising loss is
\begin{equation*}
    \mathcal{L}_{\mathrm{phase}}
    =
    \mathbb{E}_{x_0,\varepsilon,n,r_0}
    \left[
    \left|
        \varepsilon
        -
        \varepsilon_{\theta}(x_n,n,\phi(r_n))
    \right|^2
    \right],
    \qquad
    r_n=r_0+n\ (\mathrm{mod}\,P) .
\end{equation*}
This term trains the model to recover the reverse dynamics at each phase. For
the invariant-average regularization, we use centered periodic forcing satisfying
$P^{-1}\sum_{s=0}^{P-1}b_s=0$; the data-adapted construction below enforces this condition.
The regularization is a shared-noise phase-average penalty motivated by Lemma~\ref{lem:inv}:
\begin{equation}\label{eqn:avgloss}
    \mathcal{L}_{\mathrm{avg}}
    =
    \mathbb{E}_{x_0,\varepsilon,n,r_0}
    \left[
    \left|
        \frac{1}{P}
        \sum_{s=0}^{P-1}
        \varepsilon_\theta(x_n,n,\phi(s))
        -
        \varepsilon
    \right|^2
    \right],
\end{equation}
Although the reverse model is phase-conditioned, the phase-wise terminal laws
should remain coupled through the averaged reference law $\bar{\mu}^\ast$.
The full objective is
\begin{equation}\label{eqn:fullloss}
    \mathcal{L}
    =
    \mathcal{L}_{\mathrm{phase}}
    +
    \lambda\mathcal{L}_{\mathrm{avg}},
    \qquad
    \lambda>0 .
\end{equation}
The first term learns the phase-specific reverse dynamics. The second term couples the phases by enforcing that their average recovers the invariant-average reverse structure. This distinguishes the proposed method from a standard diffusion model with an additional phase embedding. The derivation is deferred to Appendix \ref{ssec:derivation}.

\paragraph{Data-adapted periodic forcing.}
For manifold-supported data, the periodic forcing can be constructed from
empirical phase centers. Suppose each training sample $x_i$ is assigned a phase
$r_i\in\{0,\ldots,P-1\}$. Let
\begin{equation}
    c_r
    =
    \frac{1}{|\mathcal{I}_r|}
    \sum_{i\in\mathcal{I}_r} x_i,
    \qquad
    \mathcal{I}_r=\{i:r_i=r\}.
    \label{eq:empirical_phase_centres}
\end{equation}
To ensure compatibility with the invariant-average principle, we work in
centered coordinates. Let
\begin{equation*}
    \bar{c}
    =
    \frac{1}{P}\sum_{r=0}^{P-1} c_r,
    \qquad
    \tilde{c}_r=c_r-\bar{c}.
\end{equation*}
We then define
\begin{equation}
    b_r
    =
    \tilde{c}_{r+1}-\rho \tilde{c}_r,
    \qquad
    b_{r+P}=b_r .
    \label{eq:data_adapted_forcing}
\end{equation}
 This choice makes the deterministic part of the forward dynamics transport
centered phase centers cyclically: if $x_k=\tilde{c}_r$, then
$\rho x_k+b_r=\tilde{c}_{r+1}$. Moreover,
\begin{equation}
    \frac{1}{P}\sum_{r=0}^{P-1} b_r=0,
    \label{eq:zero_mean_forcing}
\end{equation}
which is the condition needed for the invariant-average regularization. In particular, the data-adapted forcing defined in
Eq.~\eqref{eq:data_adapted_forcing} satisfies the assumptions of Theorem \ref{thm:limit}
whenever the centered empirical phase centers are not all identical.

The algorithms below summarize the implementation of this construction.
Algorithm~\ref{alg:ptl_forward} uses the closed-form forward marginal from
Lemma~\ref{lem:forward}, so noisy samples can be generated directly at an
arbitrary diffusion depth without iterating through all intermediate steps.
Algorithm~\ref{alg:ptl_sampling} implements the learned reverse chain. At each
reverse step, the network is conditioned on the phase of the current noisy
state, while the deterministic forcing term corresponds to the phase of the
previous forward transition.

\begin{algorithm}[t]
\caption{PTL forward noising process}
\label{alg:ptl_forward}
\begin{algorithmic}[1]
\Require Training data $\{x_i,r_i\}_{i=1}^{N}$, period $P$, contraction
$\rho\in(0,1)$, maximum diffusion depth $K$.

\State Estimate phase centers $c_r$ via Eq.~\eqref{eq:empirical_phase_centres}.

\State Define periodic forcing $b_r$ via Eq.~\eqref{eq:data_adapted_forcing}.

\State Precompute offsets
\begin{equation*}
    a_{0,r}=0,
    \qquad
    a_{k+1,r}
    =
    \rho a_{k,r}+b_{r+k},
    \qquad
    k=0,\ldots,K-1 .
\end{equation*}

\State Sample $(x_0,r_0)$ from the training set, $k\sim \mathrm{Unif}\{1,\ldots,K\}$,
and $\varepsilon\sim \mathcal{N}(0,I_d)$.

\State Generate the noisy sample by the direct PTL marginal
\begin{equation*}
    x_k
    =
    \rho^k x_0
    +
    a_{k,r_0}
    +
    \sqrt{1-\rho^{2k}}\,\varepsilon .
\end{equation*}

\State Set the current phase
\begin{equation*}
    r_k = r_0+k\ (\mathrm{mod}\,P) .
\end{equation*}

\State \Return $(x_k,r_k,\varepsilon)$.
\end{algorithmic}
\end{algorithm}
Algorithm~\ref{alg:ptl_forward} is justified by the closed-form marginal
\[
    x_k
    =
    \rho^k x_0
    +
    \sum_{j=0}^{k-1}\rho^{k-1-j}b_{r_0+j}
    +
    \sqrt{1-\rho^{2k}}\,\varepsilon,
\]
where we use the normalized choice $\sigma=\sqrt{1-\rho^2}$. The precomputed
offset $a_{k,r}$ is exactly the deterministic forcing contribution
\[
    a_{k,r}
    =
    \sum_{j=0}^{k-1}\rho^{k-1-j}b_{r+j}.
\]
Thus the algorithm samples exactly from $q(x_k\mid x_0,r_0)$.

\begin{algorithm}[t]
\caption{Sampling from PTL-Diffusion}
\label{alg:ptl_sampling}
\begin{algorithmic}[1]
\Require Trained denoiser $\varepsilon_\theta$, period $P$, diffusion depth $K$,
contraction $\rho$, periodic forcing $\{b_r\}_{r=0}^{P-1}$, reverse variance
schedule $\{\tilde{\sigma}_k^2\}_{k=0}^{K-1}$.
\State Sample an initial phase $r_0\sim \mathrm{Unif}\{0,\ldots,P-1\}$ and set
$r_K=r_0+K\ (\mathrm{mod}\,P)$.
\State Sample $x_K$ from the finite-depth terminal approximation
\[
    x_K\sim
    \mathcal{N}
    \left(
        a_{K,r_0},
        (1-\rho^{2K})I_d
    \right)
\]
with the terminal phase $r_K=r_0+K\ (\mathrm{mod}\,P)$.

\For{$k=K-1,K-2,\ldots,0$}
    \State Set $r_k = r_K - (K-k)\ (\mathrm{mod}\,P)$.
    \State Predict noise
    \begin{equation*}
        \hat{\varepsilon}
        =
        \varepsilon_\theta(x_{k+1},k+1,\phi(r_{k+1})).
    \end{equation*}
    \State Compute the reverse mean
    \begin{equation*}
        \mu_\theta(x_{k+1},k+1,r_{k+1})
        =
        \frac{1}{\rho}
        \left(
            x_{k+1}-b_{r_k}
        \right)
        -
        \kappa_k\hat{\varepsilon},
    \end{equation*}
    where $\kappa_k$ is the scalar coefficient from the Gaussian reverse posterior.
    \State Sample
    \begin{equation*}
        x_k
        =
        \mu_\theta(x_{k+1},k+1,r_{k+1})
        +
        \tilde{\sigma}_k z,
        \qquad
        z\sim \mathcal{N}(0,I_d),
    \end{equation*}
    with $z=0$ at the final step if deterministic sampling is used.
\EndFor
\State \Return $x_0$.
\end{algorithmic}
\end{algorithm}
Algorithm~\ref{alg:ptl_sampling} follows the Gaussian reverse parameterization
from Lemma~\ref{lem:reverse_posterior}. In the reverse step from $x_{k+1}$ to
$x_k$, the noisy input $x_{k+1}$ has phase $r_{k+1}$, so the denoiser is
conditioned on $\phi(r_{k+1})$. However, the forward transition from $x_k$ to
$x_{k+1}$ used the forcing $b_{r_k}$; hence the reverse mean subtracts
$b_{r_k}$. The coefficient $\kappa_k$ is the scalar multiplying the predicted noise
in the Gaussian reverse posterior, and $\tilde{\sigma}_k^2$ is the chosen
reverse variance.
\section{Experiments}\label{sec:experiment}

\subsection{Overview}

We evaluate PTL-Diffusion as a diffusion model for manifold-structured data. The
central hypothesis is that replacing the single Gaussian terminal law of DDPM by
a phase-indexed periodic terminal family gives the forward process a better
coarse geometry for data concentrated near a low-dimensional manifold.

We consider two experimental settings. First, we use synthetic manifold data,
where the phase coordinate and target geometry are known, to verify that PTL-Diffusion
recovers phase-conditioned structure and the invariant-average law. Second, we
evaluate on the Olivetti face dataset, where images lie near a low-dimensional
eigenface manifold. 

Within each dataset, DDPM, phase-conditioned DDPM, PTL-Diffusion, and
PTL-Diffusion without invariant-average regularization use the same denoising
backbone and training budget. Thus performance differences can be attributed to
the forward reference law, phase conditioning, and invariant-average
regularization rather than to model capacity.

We conducted all experiments on a server running Ubuntu 24.04.4 LTS 64-bit. 
The server features 376 GiB of memory and two Intel\textregistered{} Xeon\textregistered{} Silver 4108 CPUs @ 1.80 GHz with 16 physical cores and 32 logical CPUs, and an NVIDIA A100-PCIE-40GB GPU, providing a robust computational environment for the applications. 
The implementation was based on PyTorch 2.6.0 with CUDA 12.4 support.
\subsection{Models and ablations}
We compare the following models.

\begin{enumerate}
    \item DDPM. A standard diffusion baseline whose forward noising
    process converges to a single time-homogeneous Gaussian terminal reference
    law.
    \item Phase-conditioned DDPM (Phase-DDPM). A diffusion baseline with the same
    standard forward noising process as DDPM, so that its terminal reference
    law remains a single time-homogeneous Gaussian distribution. Unlike DDPM,
    its denoising network also receives the sinusoidal phase embedding as an
    additional input. This baseline tests whether any improvement comes merely
    from phase conditioning in the reverse model, rather than from the periodic
    terminal law used in PTL-Diffusion.
    \item PTL-Diffusion. The full proposed model uses the periodic
    forward law, sinusoidal phase embedding, and invariant-average
    regularization in \eqref{eqn:avgloss}. This is the main model evaluated in our proof-of-concept
    experiments.
    \item PTL-Diffusion without invariant-average regularization (PTL-Diffusion w/o reg.). The
    model uses both the periodic forward law and the phase embedding, but
    removes $\mathcal{L}_{\mathrm{avg}}$ defined in \eqref{eqn:avgloss}. This tests whether the phase-wise
    reverse models become unstable or incoherent without the invariant-average
    principle.
\end{enumerate}
The ablation study compares the last two models above.
\subsection{Manifold-based sample generation}
\label{subsec:manifold_sample_generation}

We first evaluate PTL-Diffusion on synthetic datasets supported near simple
embedded manifolds. These experiments are designed as controlled
proof-of-concept tests: the target geometry is known, an intrinsic cyclic
coordinate is available, and we can directly assess whether generated samples
remain close to the prescribed manifold. 
\paragraph{Torus point-cloud benchmark.}
For the torus benchmark, we generate i.i.d. samples on the embedded torus
$\mathbb{T}^2=S^1\times S^1\subset\mathbb{R}^4$. We sample
$\theta^{(1)}\sim \mathrm{Unif}(0,2\pi)$ and set
\[
    \theta^{(2)} = 2\theta^{(1)} + \xi \pmod{2\pi},
    \qquad \xi \sim \mathcal{N}(0,\sigma_\theta^2),
\]
before embedding each sample as
\[
    X
    =
    \left(
        \cos \theta^{(1)},
        \sin \theta^{(1)},
        \cos \theta^{(2)},
        \sin \theta^{(2)}
    \right)
    + \eta,
    \qquad \eta \sim \mathcal{N}(0,\sigma_x^2 I_4).
\]
The phase label is obtained by binning $\theta^{(1)}$ into $P$ cyclic phase bins.

\paragraph{Cylinder point-cloud benchmark.}
For the cylinder benchmark, we generate i.i.d. samples near the embedded
cylindrical surface
$S^1\times[-h/2,h/2]\subset\mathbb{R}^3$. We sample the angular and height
variables as
\[
    \theta \sim \mathrm{Unif}(0,2\pi),
    \qquad
    z \sim \mathrm{Unif}\left(-\frac{h}{2}, \frac{h}{2}\right),
\]
where $h$ denotes the cylinder height. Each sample is then embedded into
$\mathbb{R}^3$ by
\[
    X = (\cos\theta, \sin\theta, z) + \eta,
    \qquad
    \eta \sim \mathcal{N}(0,\sigma_x^2 I_3).
\]
The phase label is obtained by binning the $S^1$ coordinate $\theta$ into
$P$ cyclic phase bins.
\paragraph{Evaluation.}
The evaluation focuses on three aspects:
ambient distributional fidelity, manifold consistency, and phase recovery.

First, we report the coordinate-wise marginal Wasserstein-$1$ distance
\begin{equation*}
    E_{\mathcal{W}_1}
    =
    \frac{1}{d}
    \sum_{j=1}^{d}
    \mathcal{W}_1
    \left(
        \{x_i^{(j)}\}_{i=1}^{N},
        \{\hat{x}_i^{(j)}\}_{i=1}^{M}
    \right),
\end{equation*}
where $\{x_i\}_{i=1}^N$ are real samples and $\{\hat{x}_i\}_{i=1}^M$ are
generated samples. This measures marginal fidelity in the ambient Euclidean
coordinates.

Second, we compare dependence structure through the empirical correlation
error
\begin{equation*}
    E_{\mathrm{corr}}
    =
    \frac{1}{d^2}
    \left\|
        C_{\mathrm{real}} - C_{\mathrm{gen}}
    \right\|_1 ,
\end{equation*}
where $C_{\mathrm{real}}$ and $C_{\mathrm{gen}}$ are the empirical correlation
matrices of the real and generated samples, respectively, and
$\|A\|_1$ denotes the entrywise matrix $\ell^1$ norm, i.e.,
\[
    \|A\|_1=\sum_{i=1}^{d}\sum_{j=1}^{d}|A_{ij}|.
\]

Third, because the true manifolds are known, we directly measure manifold
constraint violations. For the torus, we define
\begin{equation*}
    E_{\mathrm{torus}}
    =
    \mathbb{E}_{\hat{x}}
    \left[
        \left|
            (\hat{x}_1^2+\hat{x}_2^2)-1
        \right|
        +
        \left|
            (\hat{x}_3^2+\hat{x}_4^2)-1
        \right|
    \right].
\end{equation*}
For the cylinder, we define
\begin{equation*}
    E_{\mathrm{cyl}}
    =
    \mathbb{E}_{\hat{x}}
    \left[
        \left|
            (\hat{x}_1^2+\hat{x}_2^2)-1
        \right|
    \right].
\end{equation*}
These quantities measure whether generated samples remain close to the
embedded torus or cylinder.

Finally, we evaluate phase recovery. For a generated torus sample, we infer the
first circular phase by
\begin{equation*}
    \hat{r}(\hat{x})
    =
    \left\lfloor
    \frac{P}{2\pi}
    \operatorname{atan2}(\hat{x}_2,\hat{x}_1)
    \right\rfloor
    \pmod{P} .
\end{equation*}
Let $p_{\mathrm{real}}(r)$ and $p_{\mathrm{gen}}(r)$ be the empirical phase
histograms. We report
\begin{equation*}
    E_{\mathrm{phase\text{-}TV}}
    =
    \frac{1}{2}
    \sum_{r=0}^{P-1}
    \left|
        p_{\mathrm{real}}(r)-p_{\mathrm{gen}}(r)
    \right|.
\end{equation*}
We also report the phase-conditioned distributional error
\begin{equation*}
    E_{\mathrm{phase}}
    =
    \frac{1}{P}
    \sum_{r=0}^{P-1}
    D(\nu_r,\hat{\nu}_r),
\end{equation*}
where $\nu_r$ and $\hat{\nu}_r$ are the real and generated laws conditioned on
phase $r$. In practice, $D$ is implemented as a moment-based or Wasserstein
distance in the observed embedding space.

The invariant-average principle is evaluated by
\begin{equation*}
    E_{\mathrm{avg}}
    =
    D
    \left(
        \frac{1}{P}\sum_{r=0}^{P-1}\nu_r,
        \frac{1}{P}\sum_{r=0}^{P-1}\hat{\nu}_r
    \right).
\end{equation*}
This metric directly tests whether the averaged generated phase family matches
the averaged real reference law.
\paragraph{Experiment setup.} For both datasets, we compare DDPM, Phase-DDPM, PTL-Diffusion without regularization, and full PTL-Diffusion. We use the same denoising backbone and
training budget for all models. For each dataset, 12000 samples were used for training and 3000 samples were used for evaluation. 
All models were trained for 60 epochs under the same training configuration to ensure a fair comparison. For PTL-diffusion and its variants, the periodic parameter is set to $P=30$. 
After training, generated samples were drawn from each model and compared with the target data distribution both visually and quantitatively. 
For the quantitative evaluation, the reported results were computed over 10 independent sampling runs and presented as mean $\pm$ standard deviation. Lower metric values indicate a closer match between the generated and target distributions. 
\paragraph{Model comparison.}
Fig.~\ref{fig:torus_cylinder} shows the generated samples on the torus and cylinder datasets after 60 training epochs. 
For both datasets, the conventional DDPM fails to preserve the underlying periodic geometry and tends to generate samples concentrated around a non-periodic Gaussian-like cloud. 
By introducing phase information into the denoising network, Phase-DDPM produces more structured samples than DDPM and partially recovers the circular patterns in both datasets. 
This indicates that phase conditioning itself provides useful information for learning periodic data distributions. 
However, Phase-DDPM still does not fully preserve the target manifold geometry, especially in the torus projections and the cylinder height distribution. 

In contrast, PTL-Diffusion produces samples that are much better aligned with the circular and periodic structures of the target distributions. 
On the torus dataset, PTL-Diffusion successfully recovers the ring-shaped projections on both $S^1$ components, whereas DDPM produces scattered samples without clear periodic structure. 
Similarly, on the cylinder dataset, PTL-Diffusion captures both the circular cross-section and the distribution along the cylinder height, while DDPM again loses the periodic structure and generates samples mainly around the center.

The quantitative results further support the visual observations. 
As shown in Tables~\ref{tab:torus_metrics} and~\ref{tab:cylinder_metrics}, Phase-DDPM improves upon DDPM on most synthetic metrics, confirming that adding phase conditioning to the denoising network is beneficial in this controlled setting. 
For example, on the torus dataset, $E_{\mathcal{W}_1}$ decreases from $0.159$ for DDPM to $0.109$ for Phase-DDPM, and $E_{\mathrm{phase}}$ decreases from $0.240$ to $0.167$. 
On the cylinder dataset, Phase-DDPM also reduces $E_{\mathcal{W}_1}$ from $0.148$ to $0.084$ and $E_{\mathrm{avg}}$ from $0.145$ to $0.080$. 
Nevertheless, the PTL-Diffusion variants achieve stronger overall performance, particularly on geometry-related metrics. 
For the torus dataset, PTL-Diffusion further reduces $E_{\mathcal{W}_1}$ to $0.043$, $E_{\mathrm{phase}}$ to $0.057$, and $E_{\mathrm{avg}}$ to $0.039$. 
For the cylinder dataset, PTL-Diffusion obtains a much lower cylinder constraint error than Phase-DDPM, reducing $E_{\mathrm{cyl}}$ from $0.341$ to $0.135$, and also achieves a lower overall average error of $0.076$. 
These results suggest that the improvement of the PTL-Diffusion variants cannot be explained merely by providing phase information to the neural network. 
Instead, the periodic forward formulation plays an important role in preserving the intrinsic geometry of periodic domains.

\begin{figure}[H]
\centering

\begin{subfigure}[t]{0.88\textwidth}
    \centering
    \includegraphics[width=\textwidth]{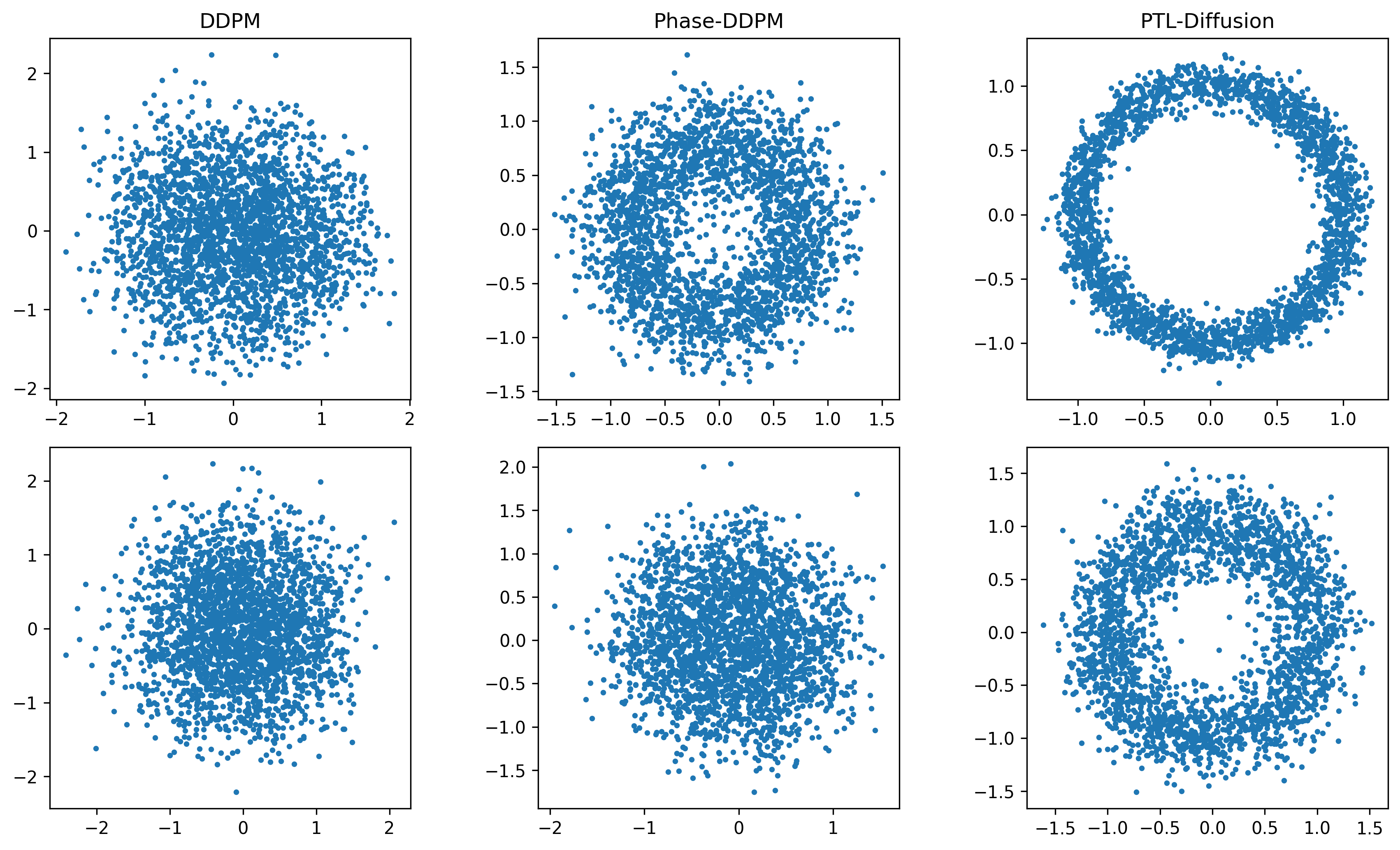}
    \caption{Torus dataset}
    \label{fig:torus_sub}
\end{subfigure}

\vspace{0.8em}

\begin{subfigure}[t]{0.88\textwidth}
    \centering
    \includegraphics[width=\textwidth]{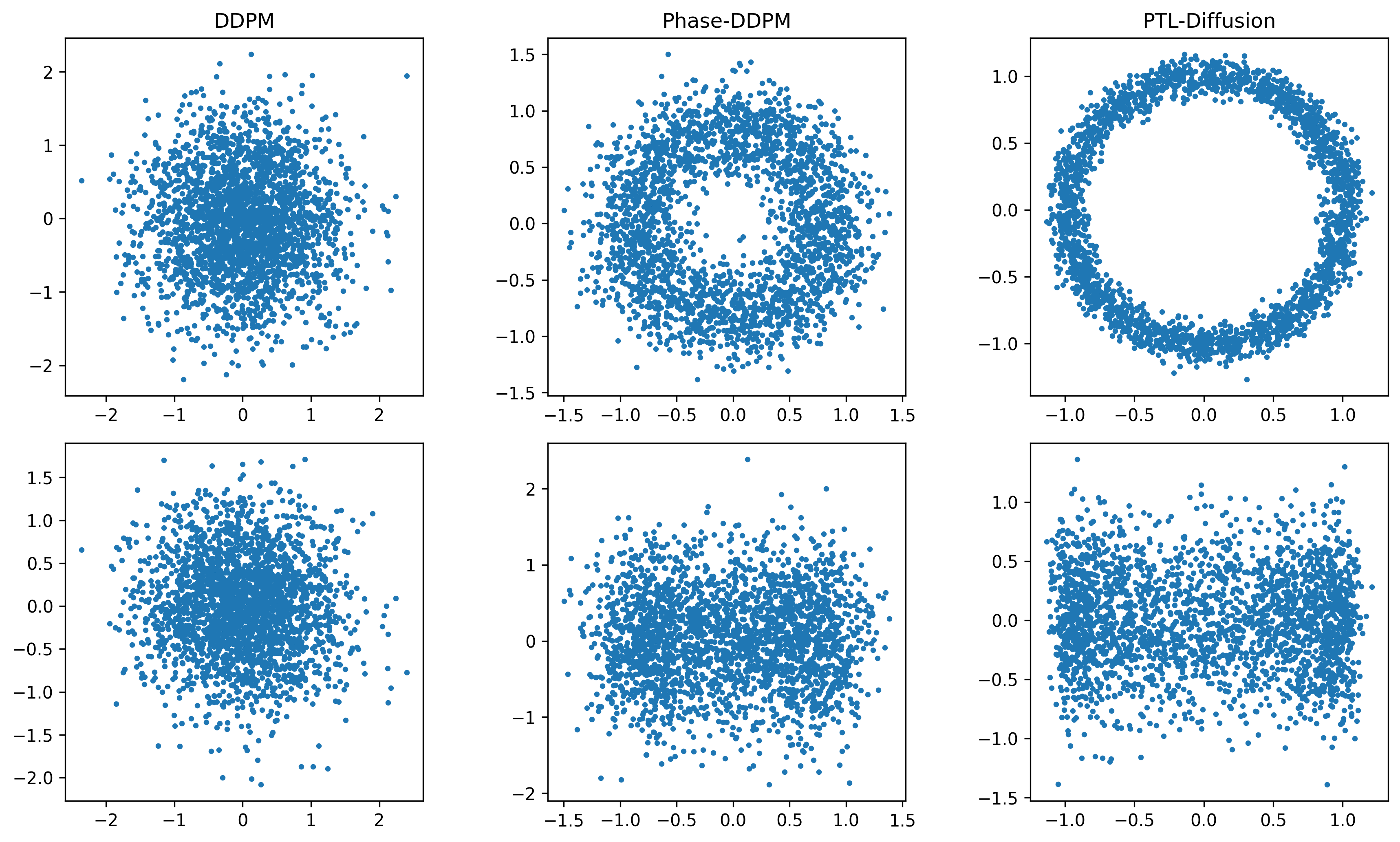}
    \caption{Cylinder dataset}
    \label{fig:cylinder_sub}
\end{subfigure}

\vspace{-0.5em}

\caption{Sample comparison on the torus and cylinder datasets after 60 training epochs. 
For the torus dataset, the columns show samples generated by DDPM and PTL-Diffusion from left to right, while the upper and lower panels show the projections onto the first and second $S^1$ components of the torus, respectively. 
For the cylinder dataset, the plots show samples generated by DDPM and PTL-Diffusion, where the upper row shows the circular cross-section and the lower row shows the distribution along the cylinder height.}
\label{fig:torus_cylinder}
\end{figure}

\begin{table}[H]
\centering
\caption{Quantitative comparison on the torus dataset over 10 independent sampling runs. 
Results are reported as mean $\pm$ standard deviation. Lower values indicate better performance for all metrics.}
\label{tab:torus_metrics}
\vspace{-0.8em}
\resizebox{\textwidth}{!}{
\begin{tabular}{lcccc}
\hline
Metric & DDPM & Phase-DDPM & PTL-Diffusion & PTL-Diffusion w/o reg. \\
\hline
$E_{\mathcal{W}_1}$ 
& $0.159 \pm 0.003$ 
& $0.109 \pm 0.002$ 
& $\mathbf{0.043 \pm 0.002}$ 
& $0.043 \pm 0.003$ \\

$E_{\mathrm{corr}}$ 
& $0.026 \pm 0.005$ 
& $\mathbf{0.015 \pm 0.004}$ 
& $0.016 \pm 0.002$ 
& $0.017 \pm 0.004$ \\

$E_{\mathrm{phase\text{-}TV}}$ 
& $0.059 \pm 0.007$ 
& $0.057 \pm 0.007$ 
& $\mathbf{0.052 \pm 0.005}$ 
& $0.056 \pm 0.007$ \\

$E_{\mathrm{torus}}$ 
& $1.319 \pm 0.015$ 
& $0.879 \pm 0.012$ 
& $0.496 \pm 0.005$ 
& $\mathbf{0.486 \pm 0.006}$ \\

$E_{\mathrm{phase}}$ 
& $0.240 \pm 0.003$ 
& $0.167 \pm 0.003$ 
& $\mathbf{0.057 \pm 0.001}$ 
& $0.057 \pm 0.001$ \\

$E_{\mathrm{avg}}$ 
& $0.159 \pm 0.003$ 
& $0.107 \pm 0.002$ 
& $0.039 \pm 0.001$ 
& $\mathbf{0.038 \pm 0.001}$ \\
\hline
\end{tabular}
}
\end{table}

\begin{table}[H]
\centering
\caption{Quantitative comparison on the cylinder dataset over 10 independent sampling runs. 
Results are reported as mean $\pm$ standard deviation. Lower values indicate better performance for all metrics.}
\label{tab:cylinder_metrics}
\vspace{-0.8em}
\resizebox{\textwidth}{!}{
\begin{tabular}{lcccc}
\hline
Metric & DDPM & Phase-DDPM & PTL-Diffusion & PTL-Diffusion w/o reg. \\
\hline
$E_{\mathcal{W}_1}$ 
& $0.148 \pm 0.003$ 
& $\mathbf{0.084 \pm 0.003}$ 
& $0.086 \pm 0.005$ 
& $0.087 \pm 0.003$ \\

$E_{\mathrm{corr}}$ 
& $0.014 \pm 0.006$ 
& $0.015 \pm 0.004$ 
& $\mathbf{0.013 \pm 0.004}$ 
& $0.014 \pm 0.004$ \\

$E_{\mathrm{phase\text{-}TV}}$ 
& $0.064 \pm 0.008$ 
& $\mathbf{0.055 \pm 0.007}$ 
& $0.056 \pm 0.007$ 
& $0.059 \pm 0.006$ \\

$E_{\mathrm{cyl}}$ 
& $0.715 \pm 0.012$ 
& $0.341 \pm 0.004$ 
& $0.135 \pm 0.002$ 
& $\mathbf{0.128 \pm 0.002}$ \\

$E_{\mathrm{phase}}$ 
& $0.188 \pm 0.003$ 
& $0.107 \pm 0.002$ 
& $0.085 \pm 0.002$ 
& $\mathbf{0.084 \pm 0.001}$ \\

$E_{\mathrm{avg}}$ 
& $0.145 \pm 0.004$ 
& $0.080 \pm 0.001$ 
& $\mathbf{0.076 \pm 0.002}$ 
& $0.077 \pm 0.002$ \\
\hline
\end{tabular}
}
\end{table}
\paragraph{Ablation study.}
The ablation results using PTL-Diffusion without regularization show that the model still performs competitively, and in some metrics it obtains slightly lower numerical values than the fully regularized version. 
However, the differences between PTL-Diffusion and its non-regularized variant are relatively small compared with the large performance gap between DDPM and the proposed periodic models. 
This suggests that the main performance gain comes from the periodic diffusion formulation itself, while the regularization term provides additional control over the learned representation and can help stabilize the training behavior.
\paragraph{Sensitivity analysis.}
{
Figure~\ref{fig:p_sensitivity} investigates the influence of the periodic parameter $P$ on the performance of PTL-Diffusion for the torus and cylinder datasets.}

{
For both datasets, the relevant manifold-constraint error ($E_{\mathrm{torus}}$ or $E_{\mathrm{cyl}}$) decreases substantially as $P$ increases from a small value, indicating that an overly coarse phase discretization is insufficient to capture the underlying periodic geometry. 
This effect is particularly pronounced for the cylinder dataset, where $E_{\mathrm{cyl}}$ drops rapidly from $P=5$ to around $P=20$ and then decreases more gradually. 
A similar trend is observed on the torus dataset, although the torus constraint error remains larger overall, reflecting the more complex periodic structure of the torus geometry. In contrast, both $E_{\mathcal{W}_1}$ and $E_{\mathrm{corr}}$ are much less sensitive to $P$.}

{
In our experiments, $P=30$ is therefore adopted as a practical default, since it lies in the stable region where further increases in $P$ provide only limited additional improvement.
}
\begin{figure}[H]
    \centering    \includegraphics[width=1.0\linewidth]{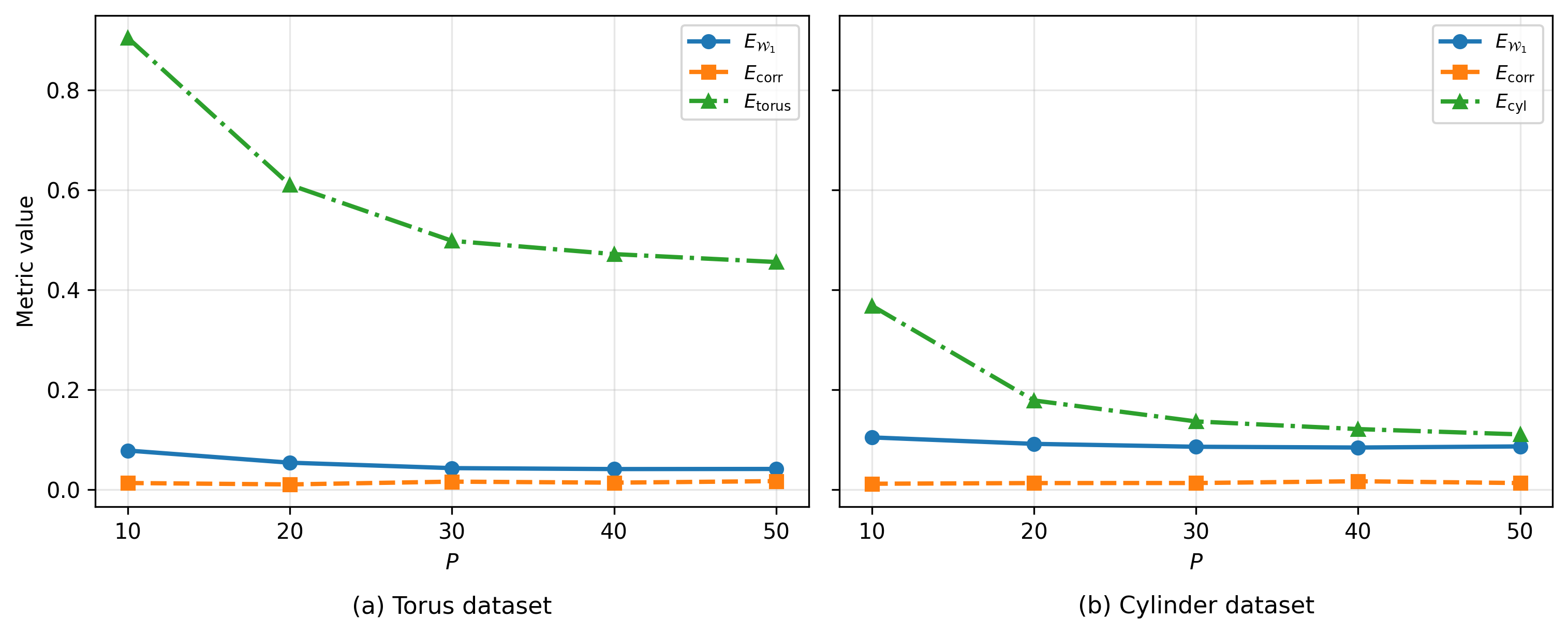}
\caption{Sensitivity analysis of PTL-Diffusion under different periodic parameters $P$: (a) reports the results on the torus dataset, while (b) reports the results on the cylinder dataset. 
For each value of $P$, the model is evaluated over 10 independent sampling runs, and the plotted values represent the mean performance.}
\label{fig:p_sensitivity}
\end{figure}

\subsection{Face dataset generation}
\label{subsec:face_generation}
We next evaluate PTL-Diffusion on the Olivetti face dataset \cite{olivetti_faces} as a small-scale
image-manifold benchmark. The dataset contains $400$ grayscale face images of
size $64\times 64$ from $40$ subjects. As shown in Figs.~\ref{fig:olivetti_diff_sub} and~\ref{fig:olivetti_same_sub}, the dataset contains both inter-subject variations across different individuals and intra-subject variations for the same person. Although the images lie in a
$4096$-dimensional ambient space, their dominant variability is well captured
by a low-dimensional eigenface/PCA representation. We therefore use this
dataset to test whether PTL-Diffusion improves generation on data concentrated
near an empirical face manifold.

\begin{figure}[H]
\centering

\begin{subfigure}[t]{1.0\textwidth}
    \centering
    \includegraphics[width=1.0\textwidth]{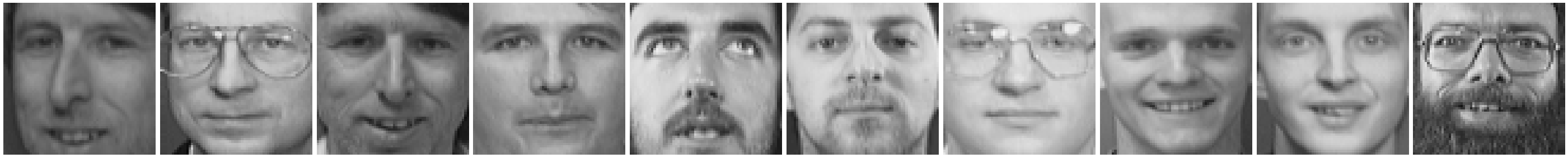}
    \vspace{-1.5em}
    \caption{Ten sample face images from different subjects}
    \label{fig:olivetti_diff_sub}
\end{subfigure}

\vspace{0.8em}

\begin{subfigure}[t]{1.0\textwidth}
    \centering
    \includegraphics[width=1.0\textwidth]{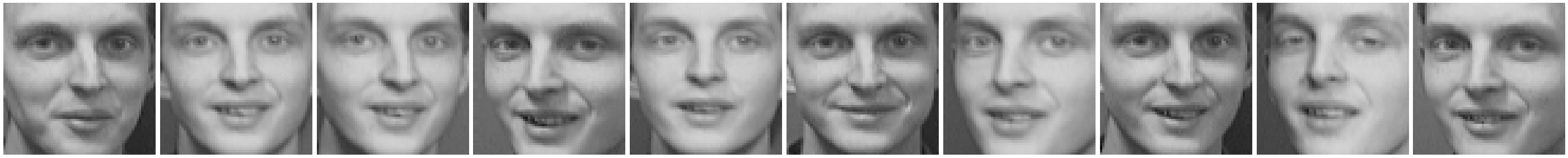}
    \includegraphics[width=1.0\textwidth]{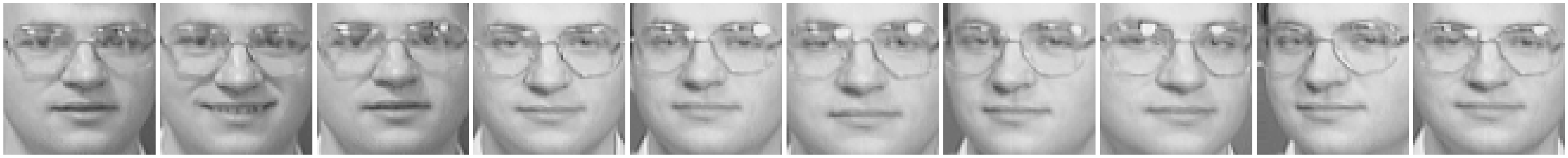}
    \includegraphics[width=1.0\textwidth]{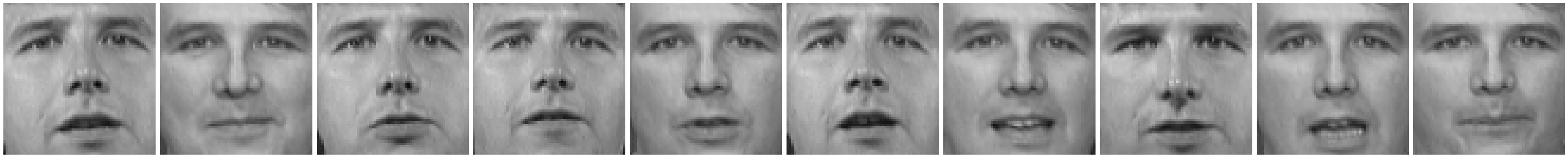}
    \vspace{-1.5em}
    \caption{Multiple images of the same subjects}
    \label{fig:olivetti_same_sub}
\end{subfigure}

\vspace{-0.5em}

\caption{
Examples from the Olivetti faces dataset:
(\subref{fig:olivetti_diff_sub}) shows ten face images selected from different subjects, illustrating inter-subject variation; 
(\subref{fig:olivetti_same_sub}) shows multiple images of the same subjects, illustrating intra-subject variation caused by changes in expression, pose, and viewing conditions.
}
\label{fig:olivetti_examples}
\end{figure}

\paragraph{Phase construction.}
Since Olivetti is not a dynamical system with a physical time phase, we
construct the phase variable from the data manifold itself. {In this experiment, the periodic parameter is set to $P=16$.} We first fit PCA on
the training images and project each image $x_i$ to its first two principal
components,
\begin{equation*}
    z_i
    =
    (z_{i,1},z_{i,2})
    =
    F_{\mathrm{PCA}}(x_i).
\end{equation*}
The phase label is then defined by the angular coordinate
\begin{equation*}
    r_i
    =
    \left\lfloor
    \frac{P}{2\pi}
    \operatorname{atan2}(z_{i,2},z_{i,1})
    \right\rfloor
    \pmod{P} .
\end{equation*}
This construction treats phase as a coarse coordinate on the empirical
eigenface manifold, rather than as physical time.

Given the phase labels, we estimate phase centers using
Eq.~\eqref{eq:empirical_phase_centres} and use the centered data-adapted forcing
from Eq.~\eqref{eq:data_adapted_forcing}. Thus the deterministic part of the
PTL forward process transports the centered phase center $\tilde c_r$ toward
$\tilde c_{r+1}$ while preserving the zero-mean forcing condition in
Eq.~\eqref{eq:zero_mean_forcing}.

\paragraph{Evaluation.}
For Olivetti, standard large-scale image metrics such as FID are less reliable
because the dataset contains only $400$ grayscale images. We therefore use
metrics tailored to small face-manifold generation.

First, we compute a pixel-level marginal Wasserstein-$1$ distance. After flattening
images into $\mathbb{R}^{4096}$, we define
\begin{equation*}
    E_{\mathrm{pixel}}
    =
    \frac{1}{|\mathcal{J}|}
    \sum_{j\in\mathcal{J}}
    \mathcal{W}_1
    \left(
        \{x_i^{(j)}\}_{i=1}^{N},
        \{\hat{x}_i^{(j)}\}_{i=1}^{M}
    \right),
\end{equation*}
where $\mathcal{J}$ is either the full set of pixel coordinates or a fixed
random subset. This measures ambient grayscale fidelity.

Second, we evaluate distributional matching in eigenface space. Let
$F_{\mathrm{PCA}}$ be the PCA map fitted on training images, and define
\begin{equation*}
    z_i = F_{\mathrm{PCA}}(x_i),
    \qquad
    \hat{z}_i = F_{\mathrm{PCA}}(\hat{x}_i),
\end{equation*}
with corresponding means $\mu_z$ and $\mu_{\hat{z}}$, and covariances $\Sigma_z$ and $\Sigma_{\hat{z}}$.
We report the eigenface mean error
\begin{equation*}
    E_{\mathrm{PCA\text{-}mean}}
    =
    \left|
        \mu_z - \mu_{\hat{z}}
    \right|,
\end{equation*}
and the eigenface covariance error
\begin{equation*}
    E_{\mathrm{PCA\text{-}cov}}
    =
    \frac{1}{d_{\mathrm{PCA}}}
    \left\|
        \Sigma_z - \Sigma_{\hat{z}}
    \right\|_F,
\end{equation*}
where $\|A\|_F$ is the Frobenius norm of a matrix $A$, and $d_{\mathrm{PCA}}\in \mathbb{N}$ is the dimension of the PCA space.
These metrics measure whether generated samples match the location and shape
of the empirical face manifold.

Third, we compute the nearest-neighbor distance from generated samples to the
training set in PCA space:
\begin{equation*}
    E_{\mathrm{NN}}
    =
    \frac{1}{M}
    \sum_{i=1}^{M}
    \min_{1\leq j\leq N_{\mathrm{train}}}
    \left|
        \hat{z}_i - z_j^{\mathrm{train}}
    \right| .
\end{equation*}
This metric acts as a small-data proxy for manifold fidelity: lower values
indicate that generated samples lie closer to the empirical face manifold.

Finally, we evaluate low-frequency illumination statistics. Let $B(\cdot)$ be a
fixed Gaussian blur or low-pass filter, and let $G_{\mathrm{PCA}}$ be a PCA map
fitted to blurred training images. We define
\begin{equation*}
    \ell_i = G_{\mathrm{PCA}}(B(x_i))\text{ and }
    \hat{\ell}_i = G_{\mathrm{PCA}}(B(\hat{x}_i)),
\end{equation*}
with corresponding means $\mu_\ell$ and $\mu_{\hat{\ell}}$, and covariances $\Sigma_\ell$ and $\Sigma_{\hat{\ell}}$.
We report
\begin{equation*}
    E_{\mathrm{LF\text{-}mean}}
    =
    \left|
        \mu_{\ell} - \mu_{\hat{\ell}}
    \right|
\text{ and }
    E_{\mathrm{LF\text{-}cov}}
    =
    \frac{1}{d_{\mathrm{LF}}}
    \left\|
        \Sigma_{\ell} - \Sigma_{\hat{\ell}}
    \right\|_F,
\end{equation*}
where $d_{\mathrm{LF}}\in \mathbb{N}$ is the dimension of the corresponding PCA space.
These metrics target coarse illumination and shading structure rather than
high-frequency pixel details.
\paragraph{Model comparison.} 
The Olivetti faces experiments were conducted for 5000 training epochs for all four models compared. Table~\ref{tab:olivetti_metrics} reports the quantitative comparison among DDPM, Phase-DDPM, PTL-Diffusion, and PTL-Diffusion without regularization. 

\begin{table}[H]
\centering
\caption{Quantitative comparison on the Olivetti faces dataset over 10 independent sampling runs. Results are reported as mean $\pm$ standard deviation. Lower values indicate better performance for all metrics.}
\label{tab:olivetti_metrics}
\vspace{-0.8em}
\resizebox{\columnwidth}{!}{
\begin{tabular}{lcccc}
\hline
Metric
& DDPM 
& Phase-DDPM
& PTL-Diffusion 
& PTL-Diffusion w/o reg. \\
\hline

$E_{\mathrm{pixel}}$ 
& $0.228 \pm 0.006$ 
& $0.233 \pm 0.006 $
& $\mathbf{0.110 \pm 0.004}$ 
& $0.122 \pm 0.004$ \\

$E_{\mathrm{PCA\text{-}mean}}$ 
& $11.338 \pm 0.174$ 
& $12.120 \pm 0.161 $
& $\mathbf{5.707 \pm 0.278}$ 
& $6.881 \pm 0.281$ \\

$E_{\mathrm{PCA\text{-}cov}}$ 
& $2.956 \pm 0.019$ 
& $2.956 \pm 0.022$
& $\mathbf{1.517 \pm 0.132}$ 
& $1.546 \pm 0.116$ \\

$E_{\mathrm{NN}}$ 
& $9.813 \pm 0.113$ 
& $10.183 \pm 0.114 $
& $7.782 \pm 0.246$ 
& $\mathbf{7.326 \pm 0.265}$ \\

$E_{\mathrm{LF\text{-}mean}}$ 
& $10.636 \pm 0.186$ 
& $11.409 \pm 0.173 $
& $\mathbf{5.292 \pm 0.304}$ 
& $6.402 \pm 0.367$ \\

$E_{\mathrm{LF\text{-}cov}}$ 
& $5.477 \pm 0.040$ 
& $ 5.474 \pm 0.045$
& $2.513 \pm 0.304$ 
& $\mathbf{2.508 \pm 0.267}$ \\

\hline
\end{tabular}
}
\end{table}
Overall, PTL-Diffusion achieves substantially lower errors than the conventional DDPM and Phase-DDPM across all evaluation metrics, indicating that the proposed periodic terminal-law formulation is also effective for structured image data. 
For example, the pixel-level error $E_{\mathrm{pixel}}$ is reduced from $0.228$ for DDPM and $0.233$ for Phase-DDPM to $0.110$ for PTL-Diffusion. 
Similarly, the PCA-based mean error decreases from $11.338$ for DDPM and $12.120$ for Phase-DDPM to $5.707$, and the LF-mean error decreases from $11.409$ to $5.292$. 
These improvements suggest that PTL-Diffusion can better capture both low-level image similarity and higher-level structural features of the face images.

{It is also worth noting that Phase-DDPM does not consistently improve over the conventional DDPM on the Olivetti dataset. 
Although it uses phase information, its errors are close to or slightly higher than those of DDPM on several metrics, such as $E_{\mathrm{pixel}}$, $E_{\mathrm{PCA\text{-}mean}}$, $E_{\mathrm{NN}}$, and $E_{\mathrm{LF\text{-}mean}}$. 
This suggests that simply adding a phase embedding to the denoising network is not sufficient to improve generation quality for this image dataset. 
In contrast, PTL-Diffusion modifies the diffusion process itself through the periodic latent formulation, which leads to a more effective representation of the structured face distribution.
}

Fig.~\ref{fig:face} further illustrates the face generation process at different training epochs. 
At the early stage of training, all models produce noisy and poorly structured outputs. 
As training progresses, DDPM gradually learns coarse facial patterns, while Phase-DDPM shows slightly more structured intermediate samples than DDPM in some epochs but remains relatively blurred. 
By comparison, PTL-Diffusion produces clearer and more face-like samples after sufficient training, with more recognizable facial contours and smoother image structures. 

\begin{figure}[H]
\centering
\includegraphics[width=1.0\textwidth]{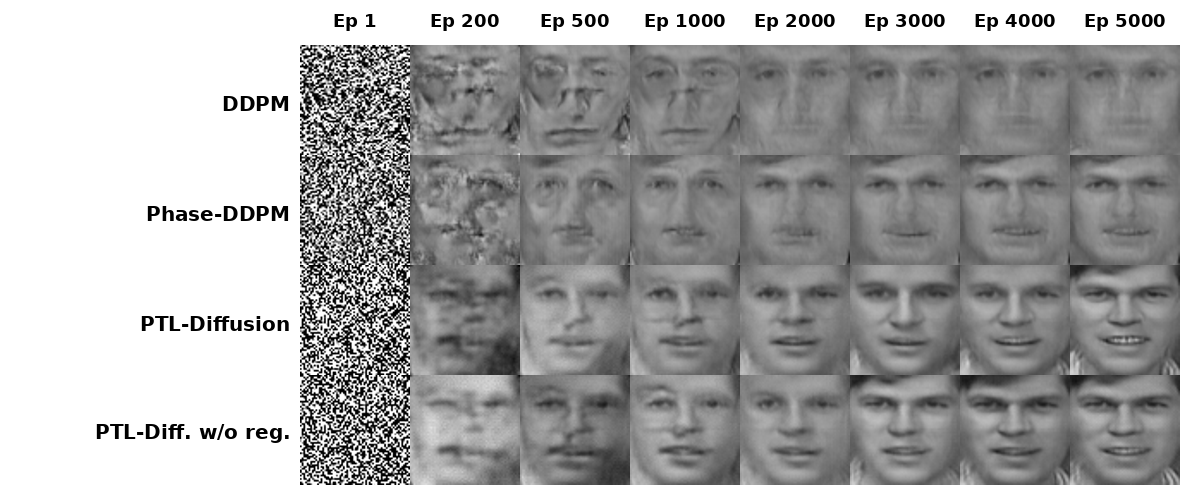}
\vspace{-2.0em}
\caption{Comparison of face generation progress: columns show selected training epochs, and the rows show DDPM, Phase-DDPM, PTL-Diffusion, and PTL-Diffusion without regularization from top to bottom.
}
\label{fig:face}
\end{figure}

\paragraph{Ablation study.}
As shown in Table~\ref{tab:olivetti_metrics}, the two variants achieve comparable performance on the Olivetti faces dataset. 
The regularized PTL-Diffusion model obtains slightly lower errors on $E_{\mathrm{pixel}}$, $E_{\mathrm{PCA\text{-}mean}}$, $E_{\mathrm{PCA\text{-}cov}}$, and $E_{\mathrm{LF\text{-}mean}}$, while the variant without regularization gives marginally lower values on $E_{\mathrm{NN}}$ and $E_{\mathrm{LF\text{-}cov}}$. 

\section{Discussion and Conclusion}

We introduced PTL-Diffusion, a diffusion framework in which the forward noising process converges to a nonconstant periodic family of Gaussian terminal laws rather than to a single time-homogeneous Gaussian reference distribution. The motivation is to provide the reverse model with a structured terminal reference geometry when the data are concentrated near a low-dimensional manifold or admit a meaningful coarse phase coordinate. In this sense, the phase variable is not used merely as an auxiliary conditioning input to the denoising network; it is built into the reference dynamics of the diffusion process itself.

The proposed construction can be viewed as a coarse, chart-inspired way of organizing the terminal reference law. A manifold may require several local coordinate descriptions rather than a single global Euclidean coordinate system. PTL-Diffusion does not explicitly learn manifold charts, transition maps, or a Riemannian metric. Instead, it uses a finite phase-indexed family of Gaussian terminal laws as a tractable approximation to this local organization. The phase variable acts as a coarse descriptor of variation in the data, while the periodic terminal family provides different reference laws for different phase classes.

The proposed construction preserves much of the tractability of standard denoising diffusion models. For the periodically forced Ornstein--Uhlenbeck-type forward process, we derived closed-form forward marginals, identified the limiting periodic Gaussian family, and obtained explicit Gaussian reverse posteriors. These formulae allow the model to be trained using a standard noise-prediction objective, while replacing the usual single terminal law by a phase-indexed reference family. We also introduced an invariant-average regularization term, which couples the phase-conditioned reverse dynamics by enforcing consistency with the averaged periodic reference law. This regularization allows phase-wise flexibility while discouraging the learned phase-conditioned reverse processes from becoming unrelated.

A key distinction from phase-conditioned DDPM is that PTL-Diffusion changes the forward reference law, not only the input variables of the reverse model. In a phase-conditioned DDPM, the phase label may enter the denoising network, but the forward process still destroys all samples toward the same time-homogeneous Gaussian terminal distribution. By contrast, PTL-Diffusion uses a forward process whose limiting object is the periodic family $\{\mu_r^\ast\}_{r=0}^{P-1}$. Thus, phase information is encoded in the noising dynamics and terminal geometry, rather than being left entirely for the neural network to recover during reverse denoising. This provides a stronger structural inductive bias, especially when the data admit a meaningful coarse phase representation.

Our experiments provide a proof-of-concept validation of this principle. On synthetic torus and cylinder benchmarks, PTL-Diffusion improves manifold-level distributional matching and phase-conditioned recovery compared with DDPM and phase-conditioned DDPM baselines under matched denoising architectures. On the Olivetti face benchmark, the method also shows improved feature-level and manifold-neighbourhood metrics, suggesting that periodic terminal laws can act as a useful inductive bias beyond explicitly periodic sample paths.

The main limitation of the present construction is that the period $P$ is fixed and finite. This means that the time-indexed or phase-indexed family of laws is represented through finitely many phase classes, which may be restrictive for strongly aperiodic data, continuously drifting nonstationary distributions, or manifolds whose latent geometry cannot be well approximated by a coarse cyclic coordinate. A related limitation is that the phase descriptor used in the current experiments is fixed before training, for example by an angular coordinate or a low-dimensional embedding. Therefore, the proposed periodic terminal family should be understood as a controlled structural approximation rather than a universally expressive model of nonstationarity or manifold geometry.

Nevertheless, these restrictions are also what make the framework mathematically transparent and computationally tractable. The finite-period construction gives a concrete setting in which the terminal reference law is no longer a single invariant Gaussian, while the forward and reverse distributions remain analytically tractable. More expressive variants could learn the phase coordinate, adapt the period $P$, use multiple or hierarchical periods, replace the Gaussian periodic family by richer structured terminal laws, or combine the proposed construction with modern U-Net, transformer, or latent-diffusion backbones.

A natural next direction is large-scale image generation, including face datasets such as CelebA, where phase variables may correspond to coarse factors such as pose, illumination, expression, or identity-related manifold coordinates. Such experiments would test whether the benefits observed in the present proof-of-concept setting persist when PTL-Diffusion is integrated with stronger architectures and larger datasets. They would also clarify when a finite phase partition is sufficient and when a more flexible learned or continuous phase representation is needed.

\bibliographystyle{plain}

\appendix
\section{Proofs}

\subsection{Proof of Lemma \ref{lem:forward}}\label{ssec:prooflem1}

\begin{proof}
Recall the phase-indexed forward recursion
\begin{equation*}
    x_{n+1}
    =
    \rho x_n
    +
    b_{r_0+n}
    +
    \sigma \varepsilon_{n+1},
\end{equation*}
where $\varepsilon_n\sim\mathcal{N}(0,I_d)$ are independent and the phase
indices are understood modulo $P$.

Unrolling the recursion gives
\begin{equation*}
    x_n
    =
    \rho^n x_0
    +
    \sum_{j=0}^{n-1}
    \rho^{n-1-j} b_{r_0+j}
    +
    \sigma
    \sum_{j=1}^{n}
    \rho^{n-j}\varepsilon_j .
\end{equation*}
Conditional on $(x_0,r_0)$, the first two terms are deterministic and the last
term is a linear combination of independent Gaussian random variables. Hence
$x_n\mid x_0,r_0$ is Gaussian:
\begin{equation*}
    q(x_n\mid x_0,r_0)
    =
    \mathcal{N}(m_{n,r_0}(x_0),\Sigma_n),
\end{equation*}
with
\begin{equation*}
    m_{n,r_0}(x_0)
    =
    \rho^n x_0
    +
    \sum_{j=0}^{n-1}
    \rho^{n-1-j} b_{r_0+j},
\end{equation*}
and
\begin{align*}
   \Sigma_n
    =
    \sigma^2
    \sum_{j=1}^{n}
    \rho^{2(n-j)} I_d  =
    \sigma^2
    \sum_{\ell=0}^{n-1}
    \rho^{2\ell} I_d  =
    \sigma^2
    \frac{1-\rho^{2n}}{1-\rho^2}I_d .
\end{align*}
This proves the claimed forward marginal.
\end{proof}
\subsection{Proof of Theorem \ref{thm:limit}}\label{ssec:proofthm1}
\begin{proof}
The argument follows the same pull-back idea as the periodically forced
Ornstein--Uhlenbeck example in \cite{wuyuan2023galerkin}: the
limiting object is obtained by starting the recursion in the remote past and
letting the initial condition be forgotten.
We first use the absolute phase convention $r_0=0$; a general initial phase is
obtained by replacing $b_k$ with the shifted forcing $b_{r_0+k}$ throughout.

Extend the forcing periodically to all integer times by setting
$b_{r+P}=b_r$. For $m<n$, write the solution started from $x_m$ at time $m$ as
\begin{equation*}
    x_n^{m,x_m}
    =
    \rho^{n-m}x_m
    +
    \sum_{k=m}^{n-1}\rho^{n-1-k}b_k
    +
    \sigma\sum_{k=m}^{n-1}\rho^{n-1-k}\varepsilon_{k+1}.
\end{equation*}
Since $0<\rho<1$, the contribution of the initial condition vanishes as
$m\to-\infty$. The deterministic series converges absolutely because $(b_k)$ is bounded and
$\sum_{j\geq 0}\rho^j<\infty$, while the Gaussian series converges in $L^2$
because its tail variance is
$\sigma^2 d\sum_{j\geq M}\rho^{2j}\to 0$. Hence the pull-back limit
\begin{equation*}
    x_n^\ast
    :=
    \sum_{j=0}^{\infty}\rho^j b_{n-1-j}
    +
    \sigma\sum_{j=0}^{\infty}\rho^j\varepsilon_{n-j}
\end{equation*}
is well defined. Its law depends on $n$ only through the phase
$r=n\,(\mathrm{mod}\,P)$.
Therefore define
\begin{equation*}
    \mu_r^\ast
    :=
    \mathcal{L}(x_n^\ast),
    \qquad r=n\,(\mathrm{mod}\,P).
\end{equation*}
Since the noise variables are i.i.d. standard Gaussian, $\mu_r^\ast$ is Gaussian
with
\begin{equation*}
    m_r^\ast
    =
    \sum_{j=0}^{\infty}\rho^j b_{r-1-j},
    \qquad
    \Sigma^\ast
    =
    \sigma^2\sum_{j=0}^{\infty}\rho^{2j}I_d
    =
    \frac{\sigma^2}{1-\rho^2}I_d .
\end{equation*}
Thus
\begin{equation*}
    \mu_r^\ast
    =
    \mathcal{N}(m_r^\ast,\Sigma^\ast).
\end{equation*}

The family is $P$-periodic. Indeed, using $b_{r+P}=b_r$,
\begin{equation*}
    m_{r+P}^\ast
    =
    \sum_{j=0}^{\infty}\rho^j b_{r+P-1-j}
    =
    \sum_{j=0}^{\infty}\rho^j b_{r-1-j}
    =
    m_r^\ast,
\end{equation*}
and $\Sigma^\ast$ is independent of $r$. Hence
$\mu_{r+P}^\ast=\mu_r^\ast$.
The family is nonconstant whenever the forcing is nonconstant: if
$m_r^\ast$ were independent of $r$, then the phase-advance mean recursion
$m_{r+1}^\ast=\rho m_r^\ast+b_r$ would force $b_r=(1-\rho)m_r^\ast$ to be
independent of $r$, a contradiction.

We next verify the phase-advance property. If $x_n^\ast\sim\mu_r^\ast$, where
$r=n\,(\mathrm{mod}\,P)$, then
\begin{equation*}
    x_{n+1}^\ast
    =
    \rho x_n^\ast+b_n+\sigma\varepsilon_{n+1}.
\end{equation*}
This follows directly from the infinite-past representation:
\begin{align*}
    \rho x_n^\ast+b_n+\sigma\varepsilon_{n+1}
    &=
    \rho
    \sum_{j=0}^{\infty}\rho^j b_{n-1-j}
    +
    b_n
    +
    \sigma\rho
    \sum_{j=0}^{\infty}\rho^j\varepsilon_{n-j}
    +
    \sigma\varepsilon_{n+1}  \\
    &=
    \sum_{j=0}^{\infty}\rho^j b_{n-j}
    +
    \sigma\sum_{j=0}^{\infty}\rho^j\varepsilon_{n+1-j}
    =
    x_{n+1}^\ast .
\end{align*}
Therefore one forward step maps $\mu_r^\ast$ to $\mu_{r+1}^\ast$, with indices
understood modulo $P$.

It remains to show attraction to this family. Let $x_0\sim\nu_0$ have finite
second moment, let $x_n^{0,x_0}$ be the process started from $x_0$ at time $0$,
and let $x_n^\ast$ be the pull-back stationary periodic version constructed
above using the same future noise variables $\varepsilon_1,\ldots,\varepsilon_n$.
Then
\begin{equation*}
    x_n^{0,x_0}-x_n^\ast
    =
    \rho^n x_0
    -
    \rho^n x_0^\ast,
\end{equation*}
where
\begin{equation*}
    x_0^\ast
    =
    \sum_{j=0}^{\infty}\rho^j b_{-1-j}
    +
    \sigma\sum_{j=0}^{\infty}\rho^j\varepsilon_{-j}.
\end{equation*}
Consequently,
\begin{equation*}
    \mathbb{E}\bigl[|x_n^{0,x_0}-x_n^\ast|^2\bigr]
    =
    \rho^{2n}
    \mathbb{E}\bigl[|x_0-x_0^\ast|^2\bigr]
    \longrightarrow 0.
\end{equation*}
By the definition of the Wasserstein-1 distance and the above coupling,
\begin{align*}
    \mathcal{W}_1
    \left(
        \mathcal{L}(x_n^{0,x_0}),
        \mu_{n\,(\mathrm{mod}\,P)}^\ast
    \right)
    &\le
    \left(
        \mathbb{E}\bigl[|x_n^{0,x_0}-x_n^\ast|^2\bigr]
    \right)^{1/2}
    \longrightarrow 0 .
\end{align*}
Thus the law of the forward process converges in $\mathcal{W}_1$ to the
corresponding phase of the periodic terminal family. In particular, it also
converges weakly along each fixed phase subsequence.

Replacing $b_k$ by the shifted forcing $b_{r_0+k}$ gives the general initial
phase statement
\begin{equation*}
    \mathcal{W}_1
    \left(
        \mathcal{L}(x_n),
        \mu_{r_0+n\,(\mathrm{mod}\,P)}^\ast
    \right)
    \longrightarrow 0 .
\end{equation*}
In particular, along any subsequence satisfying $r_0+n_\ell\equiv r\pmod{P}$,
\begin{equation*}
    \mathcal{L}(x_{n_\ell})
    \Rightarrow
    \mu_r^\ast
    \qquad
    \text{as }\ell\to\infty .
\end{equation*}

Finally, uniqueness follows from the same contraction argument. Suppose
$\{\nu_r\}_{r=0}^{P-1}$ is another $P$-periodic Gaussian family propagated by
the same transition kernels. Let $u_r$ and $\Gamma_r$ denote its phase-wise
means and covariances. After one full period,
\begin{equation*}
    \Gamma_r
    =
    \rho^{2P}\Gamma_r
    +
    \sigma^2\sum_{j=0}^{P-1}\rho^{2j}I_d,
\end{equation*}
so
\begin{equation*}
    \Gamma_r
    =
    \frac{\sigma^2}{1-\rho^2}I_d.
\end{equation*}
Similarly, comparing the mean recursion with that of $m_r^\ast$, the difference
$d_r=u_r-m_r^\ast$ satisfies
\begin{equation*}
    d_r=\rho^P d_r.
\end{equation*}
Since $0<\rho<1$, $d_r=0$. Thus $u_r=m_r^\ast$ and
$\Gamma_r=\Sigma^\ast$ for every phase $r$, proving uniqueness.
\end{proof}
\subsection{Proof of Lemma~\ref{lem:reverse_posterior}} \label{ssec:prooflem2}

\begin{proof}
Fix $x_0$ and the initial phase $r_0$. By Lemma \ref{lem:forward}, the forward marginal is
\begin{equation*}
    q(x_n\mid x_0,r_0)
    =
    \mathcal{N}
    \left(
        m_{n,r_0}(x_0),
        \Sigma_n
    \right),
\end{equation*}
where
\begin{equation*}
    \Sigma_n
    =
    \sigma^2
    \frac{1-\rho^{2n}}{1-\rho^2}I_d .
\end{equation*}
The one-step transition from $x_n$ to $x_{n+1}$ is
\begin{equation*}
    q(x_{n+1}\mid x_n,r_0)
    =
    \mathcal{N}
    \left(
        \rho x_n+b_{r_0+n},
        \sigma^2 I_d
    \right).
\end{equation*}
Equivalently,
\begin{equation*}
    q\left( \frac{x_{n+1}-b_{r_0+n}}{\rho}\right) 
    =
    \mathcal{N}\left(x_n, \frac{\sigma^2}{\rho^2} I_d\right).
\end{equation*}
Since both $q(x_n\mid x_0,r_0)$ and $q(x_{n+1}\mid x_n,r_0)$ are Gaussian,
Bayes' rule gives
\begin{align*}
    q(x_n\mid x_{n+1},x_0,r_0)
    &\propto
    q(x_{n+1}\mid x_n,r_0)
    q(x_n\mid x_0,r_0) \\
    &\propto
    \exp
    \left(
        -\frac{1}{2\sigma^2}
        \left|
            x_{n+1}-b_{r_0+n}-\rho x_n
        \right|^2
    \right) \\
    &\quad \times
    \exp
    \left(
        -\frac{1}{2}
        \left\|
            x_n-m_{n,r_0}(x_0)
        \right\|_{\Sigma_n^{-1}}^2
    \right),
\end{align*}
where $\|v\|_{A}^{2}:=v^\top A v$ for any positive definite matrix $A$.
Collecting the quadratic terms in $x_n$, the posterior precision is
\begin{equation*}
    \widetilde{\Sigma}_n^{-1}
    =
    \Sigma_n^{-1}
    +
    \frac{\rho^2}{\sigma^2}I_d .
\end{equation*}
Therefore,
\begin{align*}
    \widetilde{\Sigma}_n
   =
    \left(
        \Sigma_n^{-1}
        +
        \frac{\rho^2}{\sigma^2}I_d
    \right)^{-1} =
    \left(
        \frac{1-\rho^2}{\sigma^2(1-\rho^{2n})}I_d
        +
        \frac{\rho^2}{\sigma^2}I_d
    \right)^{-1}=
    \sigma^2
    \frac{1-\rho^{2n}}{1-\rho^{2(n+1)}}I_d .
\end{align*}
The corresponding posterior mean is
\begin{equation*}
    \widetilde{\mu}_{n,r_0}(x_{n+1},x_0)
    =
    \widetilde{\Sigma}_n
    \left(
        \Sigma_n^{-1}m_{n,r_0}(x_0)
        +
        \frac{\rho}{\sigma^2}
        \left(
            x_{n+1}-b_{r_0+n}
        \right)
    \right).
\end{equation*}
This proves the Gaussian posterior formula.

It remains to derive the equivalent noise-parameterized form. By the forward
marginal at time $n+1$,
\begin{equation*}
    x_{n+1}
    =
    m_{n+1,r_0}(x_0)
    +
    \sigma
    \sqrt{
        \frac{1-\rho^{2(n+1)}}{1-\rho^2}
    }
    \,\varepsilon,
    \qquad
    \varepsilon\sim\mathcal{N}(0,I_d).
\end{equation*}
Using the recursion for the forward mean,
\begin{equation*}
    m_{n+1,r_0}(x_0)
    =
    \rho m_{n,r_0}(x_0)
    +
    b_{r_0+n},
\end{equation*}
we get
\begin{equation*}
    \frac{1}{\rho}
    \left(
        x_{n+1}-b_{r_0+n}
    \right)
    =
    m_{n,r_0}(x_0)
    +
    \frac{\sigma}{\rho}
    \sqrt{
        \frac{1-\rho^{2(n+1)}}{1-\rho^2}
    }
    \varepsilon .
\end{equation*}

Conditional on fixed $(x_0,r_0)$, the variables $x_n$ and $x_{n+1}$ are affine
functions of the Gaussian noise vector
$(\varepsilon_1,\ldots,\varepsilon_{n+1})$. Hence
$(x_n,x_{n+1})\mid(x_0,r_0)$ is jointly Gaussian. We may
also use the Gaussian regression formula. In this case,
\begin{equation*}
    \operatorname{Cov}(x_n,x_{n+1}\mid x_0,r_0)
    =
    \rho\Sigma_n,
    \qquad
    \operatorname{Var}(x_{n+1}\mid x_0,r_0)
    =
    \Sigma_{n+1}.
\end{equation*}
Hence
\begin{equation*}
    \widetilde{\mu}_{n,r_0}(x_{n+1},x_0)
    =
    m_{n,r_0}(x_0)
    +
    \rho
    \Sigma_n
    \Sigma_{n+1}^{-1}
    \left(
        x_{n+1}-m_{n+1,r_0}(x_0)
    \right).
\end{equation*}
Substituting the noise representation of $x_{n+1}$ yields
\begin{equation*}
    \widetilde{\mu}_{n,r_0}(x_{n+1},x_0)
    =
    m_{n,r_0}(x_0)
    +
    \rho
    \frac{
        \sigma^2\frac{1-\rho^{2n}}{1-\rho^2}
    }{
        \sigma^2\frac{1-\rho^{2(n+1)}}{1-\rho^2}
    }
    \sigma
    \sqrt{
        \frac{1-\rho^{2(n+1)}}{1-\rho^2}
    }
    \varepsilon .
\end{equation*}
Equivalently,
\begin{equation*}
    \widetilde{\mu}_{n,r_0}(x_{n+1},x_0)
    =
    m_{n,r_0}(x_0)
    +
    \frac{
        \rho\sigma(1-\rho^{2n})
    }{
        \sqrt{1-\rho^2}\sqrt{1-\rho^{2(n+1)}}
    }
    \varepsilon .
\end{equation*}

Comparing this expression with
\begin{equation*}
    \frac{1}{\rho}
    \left(
        x_{n+1}-b_{r_0+n}
    \right)
    =
    m_{n,r_0}(x_0)
    +
    \frac{\sigma}{\rho}
    \sqrt{
        \frac{1-\rho^{2(n+1)}}{1-\rho^2}
    }
    \varepsilon,
\end{equation*}
we obtain
\begin{equation*}
    \widetilde{\mu}_{n,r_0}(x_{n+1},\varepsilon)
    =
    \frac{1}{\rho}
    \left(
        x_{n+1}-b_{r_0+n}
    \right)
    -
    \kappa_n\varepsilon,
\end{equation*}
where
\begin{align*}
    \kappa_n
    =
    \frac{\sigma}{\rho}
    \sqrt{
        \frac{1-\rho^{2(n+1)}}{1-\rho^2}
    }
    -
    \frac{
        \rho\sigma(1-\rho^{2n})
    }{
        \sqrt{1-\rho^2}\sqrt{1-\rho^{2(n+1)}}
    } =
    \frac{
        \sigma\sqrt{1-\rho^2}
    }{
        \rho\sqrt{1-\rho^{2(n+1)}}
    }.
\end{align*}
This proves the claimed noise-parameterized posterior mean.
\end{proof}

\subsection{Proof of Lemma~\ref{lem:inv}} \label{ssec:prooflem3}

\begin{proof}
It is enough to test invariance against an arbitrary bounded measurable function
$f:\{0,\ldots,P-1\}\times\mathbb{R}^d\to\mathbb{R}$. By definition of the
lifted kernel $\mathcal{K}$,
\begin{align*}
    \int \mathcal{K}f(r,x)\,\Pi^\ast(\mathrm{d}r,\mathrm{d}x)
    &=
    \frac{1}{P}
    \sum_{r=0}^{P-1}
    \int_{\mathbb{R}^d}
    \mathcal{K}f(r,x)\,\mu^\ast_r(\mathrm{d}x) \\
    &=
    \frac{1}{P}
    \sum_{r=0}^{P-1}
    \int_{\mathbb{R}^d}
    \int_{\mathbb{R}^d}
    f(r+1,y)\,K_r(x,\mathrm{d}y)\,\mu^\ast_r(\mathrm{d}x).
\end{align*}
Since $K_r\mu^\ast_r=\mu^\ast_{r+1}$, we have
\begin{equation*}
    \int_{\mathbb{R}^d}
    K_r(x,\mathrm{d}y)\,\mu^\ast_r(\mathrm{d}x)
    =
    \mu^\ast_{r+1}(\mathrm{d}y).
\end{equation*}
Therefore,
\begin{align*}
    \int \mathcal{K}f(r,x)\,\Pi^\ast(\mathrm{d}r,\mathrm{d}x)
    &=
    \frac{1}{P}
    \sum_{r=0}^{P-1}
    \int_{\mathbb{R}^d}
    f(r+1,y)\,\mu^\ast_{r+1}(\mathrm{d}y) \\
    &=
    \frac{1}{P}
    \sum_{s=0}^{P-1}
    \int_{\mathbb{R}^d}
    f(s,y)\,\mu^\ast_s(\mathrm{d}y),
\end{align*}
where in the last equality we relabelled $s=r+1$ modulo $P$. Hence
\begin{equation*}
    \int \mathcal{K}f\,\mathrm{d}\Pi^\ast
    =
    \int f\,\mathrm{d}\Pi^\ast .
\end{equation*}
Thus $\Pi^\ast$ is invariant under the lifted Markov kernel $\mathcal{K}$.

Finally, the marginal of $\Pi^\ast$ on $\mathbb{R}^d$ is obtained by summing out the
phase coordinate. For any Borel set $A\subseteq\mathbb{R}^d$,
\begin{equation*}
        \Pi^\ast(\{0,\ldots,P-1\}\times A)
    =
    \frac{1}{P}
    \sum_{r=0}^{P-1}
    (\delta_r\otimes\mu^\ast_r)(\{0,\ldots,P-1\}\times A)=
    \frac{1}{P}
    \sum_{r=0}^{P-1}
    \mu^\ast_r(A).
\end{equation*}
Therefore the $\mathbb{R}^d$-marginal is
$ \bar{\mu}^\ast
    =
    \frac{1}{P}
    \sum_{r=0}^{P-1}\mu_r^\ast $.
This proves the claim.
\end{proof}
\subsection{Derivation of the invariant-average regularization}
\label{ssec:derivation}
We derive the invariant-average regularization from a shared-noise phase average
of the reverse mean using the same indexing as the training objective in
Eq.~\eqref{eqn:avgloss}. This is a counterfactual phase-average calculation:
for a fixed noisy state $x_n$ and forward noise variable $\varepsilon$, we
evaluate the reverse-mean formula over all phase labels. It is not claiming that
the exact posterior for every phase has the same residual for the same observed
$x_n$. For a sample with initial phase $r_0$, the current phase at diffusion depth $n$ is
\begin{equation*}
    r_n = r_0+n\,(\mathrm{mod}\,P) .
\end{equation*}
Consider the reverse step from $x_n$ to $x_{n-1}$, where the forward transition
into $x_n$ used the forcing $b_{r_n-1}$. The learned reverse mean at current
phase $r_n$ is
\begin{equation*}
    \mu_\theta(x_n,n,r_n)
    =
    \frac{1}{\rho}
    \left(
        x_n-b_{r_n-1}
    \right)
    -
    \kappa_{n-1}
    \varepsilon_\theta(x_n,n,\phi(r_n)).
\end{equation*}

To express the invariant-average principle, we average the learned reverse mean
over all possible current phases. Using $s$ as the phase index, define
\begin{equation*}
    \bar{\mu}_\theta(x_n,n)
    :=
    \frac{1}{P}
    \sum_{s=0}^{P-1}
    \left[
        \frac{1}{\rho}
        \left(
            x_n-b_{s-1}
        \right)
        -
        \kappa_{n-1}
        \varepsilon_\theta(x_n,n,\phi(s))
    \right].
\end{equation*}
Therefore,
\begin{equation*}
    \bar{\mu}_\theta(x_n,n)
    =
    \frac{1}{\rho}x_n
    -
    \frac{1}{\rho P}
    \sum_{s=0}^{P-1} b_{s-1}
    -
    \kappa_{n-1}
    \frac{1}{P}
    \sum_{s=0}^{P-1}
    \varepsilon_\theta(x_n,n,\phi(s)).
\end{equation*}
For the centered periodic forcing used in PTL-Diffusion,
\begin{equation*}
    \frac{1}{P}\sum_{s=0}^{P-1} b_{s-1}
    =
    \frac{1}{P}\sum_{s=0}^{P-1} b_s
    =0,
\end{equation*}
as in Eq.~\eqref{eq:zero_mean_forcing}. Hence
\begin{equation*}
    \bar{\mu}_\theta(x_n,n)
    =
    \frac{1}{\rho}x_n
    -
    \kappa_{n-1}
    \frac{1}{P}
    \sum_{s=0}^{P-1}
    \varepsilon_\theta(x_n,n,\phi(s)).
\end{equation*}

The corresponding shared-noise target reverse mean is
\begin{equation*}
    \widetilde{\mu}_{n-1,r_0}(x_n,\varepsilon)
    =
    \frac{1}{\rho}
    \left(
        x_n-b_{r_n-1}
    \right)
    -
    \kappa_{n-1}\varepsilon .
\end{equation*}
Averaging this target reverse mean over phases gives
\begin{equation*}
    \bar{\widetilde{\mu}}_{n-1}(x_n,\varepsilon)
    :=
    \frac{1}{P}
    \sum_{s=0}^{P-1}
    \left[
        \frac{1}{\rho}
        \left(
            x_n-b_{s-1}
        \right)
        -
        \kappa_{n-1}\varepsilon
    \right].
\end{equation*}
Using again $\frac{1}{P}\sum_{s=0}^{P-1}b_s=0$, we obtain
\begin{equation*}
    \bar{\widetilde{\mu}}_{n-1}(x_n,\varepsilon)
    =
    \frac{1}{\rho}x_n
    -
    \kappa_{n-1}\varepsilon .
\end{equation*}

Thus,
\begin{equation*}
    \bar{\mu}_\theta(x_n,n)
    -
    \bar{\widetilde{\mu}}_{n-1}(x_n,\varepsilon)
    =
    -
    \kappa_{n-1}
    \left(
        \frac{1}{P}
        \sum_{s=0}^{P-1}
        \varepsilon_\theta(x_n,n,\phi(s))
        -
        \varepsilon
    \right).
\end{equation*}
Therefore, matching averaged reverse means is equivalent, up to the scalar
factor $\kappa_{n-1}^2$, to the invariant-average denoising regularization
\begin{equation*}
    \mathcal{L}_{\mathrm{avg}}
    =
    \mathbb{E}_{x_0,\varepsilon,n,r_0}
    \left[
    \left|
        \frac{1}{P}
        \sum_{s=0}^{P-1}
        \varepsilon_\theta(x_n,n,\phi(s))
        -
        \varepsilon
    \right|^2
    \right].
\end{equation*}
\end{document}